%% file: main.tex
\documentclass[11pt]{article}

\usepackage[textwidth=6.35in,top=0.82in,bottom=0.92in]{geometry}
\usepackage{amsmath,amssymb,amsthm,mathtools}
\usepackage{booktabs,tabularx,array,multirow}
\usepackage{graphicx}
\usepackage{float}
\usepackage{placeins}
\usepackage{microtype}
\usepackage{enumitem}
\usepackage[dvipsnames]{xcolor}
\usepackage{tcolorbox}
\usepackage{tikz}
\usetikzlibrary{arrows.meta,positioning,fit,calc,shapes.geometric,backgrounds}
\definecolor{TargetBlue}{RGB}{221,235,250}
\definecolor{LoraOrange}{RGB}{255,231,199}
\definecolor{TheoryGreen}{RGB}{220,241,224}
\definecolor{OutsideGray}{RGB}{237,239,242}
\definecolor{WarningRed}{RGB}{245,220,220}
\newtcolorbox[auto counter]{examplebox}[2][]{
  colback=TargetBlue!28,
  colframe=blue!35!black,
  boxrule=0.4pt,
  arc=1pt,
  left=6pt,
  right=6pt,
  top=5pt,
  bottom=5pt,
  before skip=8pt,
  after skip=8pt,
  fonttitle=\bfseries,
  title={Example~\thetcbcounter: #2},
  #1
}
\usepackage[colorlinks=true,linkcolor=blue,citecolor=blue,urlcolor=blue]{hyperref}

\makeatletter
\long\def\@makecaption#1#2{%
  \vskip\abovecaptionskip
  {\small\leftskip=0pt\rightskip=0pt\parfillskip=0pt plus 1fil%
   \parindent=0pt\noindent\textbf{#1:} #2\par}%
  \vskip\belowcaptionskip}
\makeatother

\newtheorem{theorem}{Theorem}[section]
\newtheorem{proposition}[theorem]{Proposition}
\newtheorem{lemma}[theorem]{Lemma}
\newtheorem{corollary}[theorem]{Corollary}
\theoremstyle{definition}

\theoremstyle{remark}
\newtheorem{remark}[theorem]{Remark}

\newcommand{\R}{\mathbb{R}}
\newcommand{\E}{\mathbb{E}}
\newcommand{\KL}{\mathrm{KL}}
\newcommand{\rank}{\operatorname{rank}}
\newcommand{\softmax}{\operatorname{softmax}}
\newcommand{\diag}{\operatorname{diag}}
\newcommand{\col}{\operatorname{col}}
\newcommand{\row}{\operatorname{row}}
\newcommand{\op}{\mathrm{op}}
\newcommand{\eff}{\mathrm{eff}}
\newcommand{\core}{\mathrm{core}}
\newcommand{\MH}{\mathrm{MH}}
\newcommand{\QK}{\mathrm{QK}}
\newcommand{\F}{\mathrm{F}}
\newcommand{\trans}{\mathsf{T}}
\newcommand{\one}{\mathbf{1}}
\newcommand{\PiC}[1]{\Pi_{#1}}
\newcommand{\cmark}{\textcolor{ForestGreen}{\ensuremath{\checkmark}}}
\newcommand{\xmark}{\textcolor{BrickRed}{\ensuremath{\times}}}

\title{How Much Rank Does LoRA Need? Rank--Error Bounds for Transformer Attention}
\author{Gerard Conangla Planes\thanks{Correspondence: \href{mailto:gerardpc@gmail.com}{\texttt{gerardpc@gmail.com}}}\\[2pt]\normalsize Aily Labs}
\date{}

\begin{document}
\maketitle

\begin{abstract}
Choosing the rank of a low-rank adaptation (LoRA) update is usually an empirical task. In this paper, we provide a task-dependent theory of the approximation error achievable at each LoRA rank for Transformer attention. We fix a pretrained attention head, a target attention function, and a distribution over inputs from the downstream task, and bound the smallest expected Kullback--Leibler (KL) error achievable by a rank-$r$ query LoRA update. When target attention probabilities are bounded away from zero, we prove a lower bound of the error proportional to $\psi(\|d\|_2)$, where $d$ is the difference between candidate and target attention scores and $\psi(t)=\min\{t^2,t\}$. We also prove an unconditional upper bound $\min\{\|d\|_2^2/4,\sqrt2\|d\|_2\}$. Under explicit realizability, geometry, and moment conditions, we then bound the best rank-$r$ error between an explicit multiple of $\psi(\sqrt{T_r})$ and $\min\{T_r/4,\sqrt{2T_r}\}$, where $T_r$ is the downstream-weighted tail energy of the target update. We also provide target-Fisher bounds when candidate scores remain within a fixed range of the target scores, and an unrestricted lower bound when a subset of tokens carries most of the probability mass. These spectral bounds describe finite-score approximation. We then construct explicit families in which softmax saturation makes the rank required to match the attention function strictly smaller than the rank required to match the finite logits. Finally, we extend the analysis to fused multi-head LoRA and joint query/key updates, exposing the effects of rank sharing and query/key factorization constraints.
\end{abstract}

\section{Introduction}

LoRA~\cite{hu2022lora} makes it cheaper to fine-tune a pretrained
Transformer~\cite{vaswani2017attention} by replacing a dense weight update
with a low-rank one.  Such a low rank is both an advantage and a restriction,
reducing the number of trainable parameters and computation, but also limiting
the representational power of the adapter. In practice, rank is commonly
chosen using rules of thumb or by training several adapters of different rank
and comparing their downstream performance
\cite{hu2022lora,zhang2023adalora,schulman2025lora}.  Such a sweep reveals
which rank worked in a particular run, but it doesn't tell us whether a smaller
adapter lacked capacity or was simply harder to train.

To isolate the representation question, suppose a dense or high-rank target
adapter has already been obtained.  We then ask: for each rank budget $r$, how
closely can an adapter reproduce the target on inputs from the downstream task?
The answer cannot be read from the singular values of the target weight update
$\Delta W$ alone. For instance, a large singular direction may have no effect if inputs from
the task never activate it, while a smaller direction may matter on nearly
every input because it repeatedly changes which tokens receive attention.
The softmax function in the attention mechanism introduces two further effects.
First, adding the same constant to
every score leaves the attention probabilities unchanged.  Second, softmax
saturates, so a score error with large magnitude need not produce a proportionally large
attention error.  The rank needed to reproduce the target function can
therefore differ from the rank suggested by the raw singular values of its
weight update.

In this paper, we begin with a LoRA update to the query projection of a single attention
head. That update changes the scores assigned to the available keys, which
softmax converts into attention probabilities.
We compare these probabilities with those of a dense or high-rank target adapter, ask
for the smallest expected KL error achievable at rank \(r\), and obtain bounds on the
error, some unconditional and some conditional on explicit assumptions.
Attention KL therefore measures how much the query update changes the attention
distribution itself, before later layers can modify that change.

We stress that our results bound the attention KL described above, but we do not
provide bounds on the head output or final task loss. To do so would require additional assumptions
about the architecture, including the value and output projections, residual connections,
feed-forward network, and later layers. These components determine how a difference
in one head is propagated and may amplify, suppress, or cancel it. Without such
architecture-specific assumptions, attention KL alone does not determine output or task error.
Figure~\ref{fig:scope} shows the scope of our analysis within a Transformer block in more detail,
including extensions to fused multi-head and joint query/key LoRA.

\begin{figure}[!h]
\centering
\resizebox{0.92\textwidth}{!}{%
\begin{tikzpicture}[
  font=\small,
  >=Latex,
  box/.style={draw,rounded corners=2pt,minimum height=10mm,align=center,inner sep=4pt},
  arr/.style={->,thick},
  input/.style={box,fill=TargetBlue},
  lora/.style={box,fill=LoraOrange},
  comp/.style={box,fill=black!3},
  outside/.style={box,fill=OutsideGray,dashed},
  note/.style={draw=ForestGreen,rounded corners=2pt,fill=ForestGreen!3,
    align=left,inner sep=4pt,font=\small},
  legend/.style={box,minimum height=5mm,text width=32mm,font=\scriptsize},
  theorykey/.style={draw=ForestGreen,rounded corners=2pt,fill=ForestGreen!3,
    minimum height=5mm,text width=39mm,align=center,font=\scriptsize}
]
\node[font=\large\bfseries] (Title) at (0,0)
 {Transformer attention with LoRA: computation and theoretical scope};

\node[input,text width=46mm] (Xin) at (0,-1.25)
 {$X(u)\in\R^{n(u)\times d}$\\input token representations};

\node[lora,text width=48mm] (Q) at (-5.2,-3.35)
 {\textbf{Adapted query projection}\\
  $\widetilde Q_h=Q_0+A_h$\\[-1pt]
  $Q_h(u)=X(u)\widetilde Q_h^\trans$};
\node[lora,text width=48mm] (K) at (0,-3.35)
 {\textbf{Adapted key projection}\\
  $\widetilde K_h=K_0+B_h$\\[-1pt]
  $K_h(u)=X(u)\widetilde K_h^\trans$};
\node[comp,text width=48mm] (V) at (5.2,-3.35)
 {\textbf{Fixed value projection}\\
  $V_h(u)=X(u)V_0^\trans$};

\coordinate (fan) at (0,-2.10);
\draw[thick] (Xin.south)--(fan);
\draw[arr] (fan)-|(Q.north);
\draw[arr] (fan)--(K.north);
\draw[arr] (fan)-|(V.north);

\node[comp,text width=65mm] (Scores) at (-2.6,-5.55)
 {$S_h(u)=\beta Q_h(u)K_h(u)^\trans+\mathrm{mask}$\\
  pairwise query--key scores};
\draw[arr] (Q.south)--++(0,-6mm)-|([xshift=-13mm]Scores.north);
\draw[arr] (K.south)--++(0,-6mm)-|([xshift=13mm]Scores.north);

\node[comp,text width=65mm] (P) at (-2.6,-7.35)
 {$P_h(u)=\sigma_{\rm row}(S_h(u))$\\
  attention probabilities for head $h$};
\draw[arr] (Scores)--(P);

\node[outside,text width=65mm] (Oh) at (0,-9.30)
 {$O_h(u)=P_h(u)V_h(u)$\\
  value-weighted output of head $h$};
\draw[arr] (P.south)--++(0,-6mm)-|([xshift=-14mm]Oh.north);
\draw[arr] (V.south)|-(Oh.east);

\node[note,text width=46mm] (Compare) at (-9.0,-7.35)
 {\textbf{Single-head question.}
 For rank $r$, how closely can $P_{r,h}(u)$ reproduce target attention
 $P_{*,h}(u)$ on the downstream task?  Error is expected KL.};
\draw[ForestGreen,dashed,thick] (Compare.east)--(P.west);

\node[comp,text width=75mm] (Heads) at (0,-11.25)
 {Repeat the displayed attention computation for $h=1,\ldots,H$\\
  to obtain $O_1(u),\ldots,O_H(u)$};
\draw[arr] (Oh)--(Heads);
\node[note,text width=46mm] (MHnote) at (-9.0,-11.25)
 {\textbf{Multi-head extension.}
 For one fused rank-$r$ adapter, we bound
 $\sum_h\KL(P_{*,h}\|P_{r,h})$ before the head outputs are mixed by $W_O$.};
\draw[ForestGreen,dashed,thick] (MHnote.east)--(Heads.west);

\node[outside,text width=75mm] (WO) at (0,-13.10)
 {$\operatorname{Concat}(O_1(u),\ldots,O_H(u))W_O^\trans$\\
  concatenate heads and apply the output projection};
\node[outside,text width=75mm] (Rest) at (0,-14.90)
 {residual connection, normalization, feed-forward network, and later layers};
\draw[arr] (Heads)--(WO); \draw[arr] (WO)--(Rest);

\node[font=\scriptsize\bfseries] (Ltitle) at (-6.6,-16.65) {Legend:};
\node[input,legend,right=3mm of Ltitle] (Linput) {input activations};
\node[lora,legend,right=3mm of Linput] (Ladapt) {projection changed by LoRA};
\node[comp,legend,right=3mm of Ladapt] (Lcomp) {standard attention operation};
\node[outside,legend,below=4mm of Linput,xshift=18mm,text width=47mm]
 (Loutside) {later Transformer computation not bounded here};
\node[theorykey,right=3mm of Loutside]
 (Lnote) {theory annotation, not an operation};
\end{tikzpicture}}\par
\caption{\textbf{Transformer attention with LoRA and the scope of our analysis.}
Read from top to bottom, the diagram traces an input through the query, key,
and value projections, attention scores, row-wise softmax, and multi-head
aggregation. Our central question is how closely a rank-$r$ LoRA update can
reproduce a target attention function on inputs from the downstream task. The
main theorem treats query-only adaptation in one head; later results cover a
fused update shared across heads and simultaneous query/key updates. For an
input $u$, $X(u)$ contains $n(u)$ token representations of width $d$. For each
head index $h\in\{1,\ldots,H\}$, query/key width is $p$ and value width is
$p_v$, with
$Q_h(u),K_h(u)\in\R^{n(u)\times p}$,
$V_h(u)\in\R^{n(u)\times p_v}$, and
$S_h(u),P_h(u)\in\R^{n(u)\times n(u)}$.
Here $\sigma_{\rm row}$ applies softmax row-wise and $\beta$ is typically
$1/\sqrt p$. The pretrained projections $Q_0,K_0,V_0$ are fixed, while
$A_h$ and $B_h$ are rank-$r$ LoRA updates. The multi-head result bounds the
sum of headwise attention KL errors before $W_O$; the remaining Transformer
computation and the final model output are outside this guarantee. }
\label{fig:scope}
\end{figure}
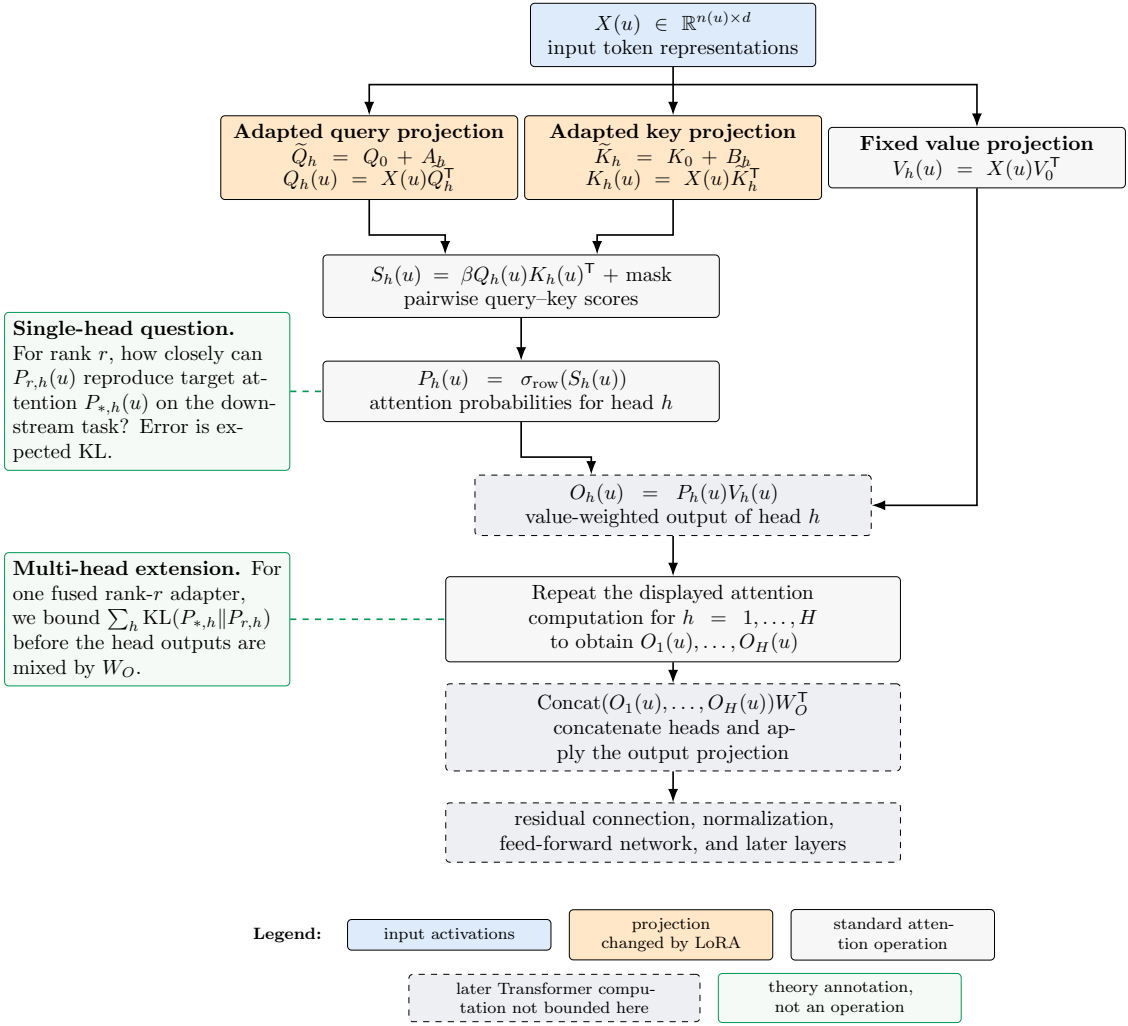

\paragraph{Overview of the results.}
For each rank budget $r$, let $\mathcal E_r$ denote the smallest expected
attention KL over all rank-$r$ candidates; Section~\ref{sec:setup}
defines this objective formally.  Our analysis first connects attention KL to
score approximation, then connects score approximation to LoRA rank.

Theorem~\ref{thm:global-softmax} gives the first connection.  If $d$ is the
difference between candidate and target scores after removing their common
shift, define
\begin{equation}
  \psi(t)=\min\{t^2,t\},
  \qquad
  \psi_{\rm up}(t)=\min\{t^2/4,\sqrt2\,t\}.
  \label{eq:psi}
\end{equation}
The theorem proves that the pointwise attention KL is comparable to
$\psi(\|d\|_2)$ when the target probabilities are bounded away from zero, and
is at most $\psi_{\rm up}(\|d\|_2)$ without this assumption.  Thus the error is
quadratic for small score differences and linear for large ones.
Corollary~\ref{cor:functional-equivalence}
lifts this pointwise statement to the expected, rank-constrained objective
$\mathcal E_r$.

Theorem~\ref{thm:main-spectral} supplies the second connection.  Under its
realizability, geometry, probability-floor, and moment assumptions, it proves
\begin{equation}
 c_{\rm lo}\,\psi(\sqrt{T_r})
 \le \mathcal E_r
 \le \psi_{\rm up}(\sqrt{T_r}).
\end{equation}
Here $T_r$ is the residual tail energy after the target update has been
weighted by the queries and keys occurring on the downstream task.  Directions
that the task never uses therefore do not contribute, and $c_{\rm lo}$ is the
explicit assumption-dependent constant in Theorem~\ref{thm:main-spectral}.
The upper bound
constructs a candidate whose error does not exceed its curve; the lower bound
applies to every candidate.  Together they bracket the best achievable error
at each rank.  Section~\ref{sec:rank-estimation} explains how to estimate this
rank--error curve from downstream data and compare it with a desired error
tolerance.

When the probability floor in the main theorem is too small to give a useful
lower bound, Section~\ref{sec:decision} provides two alternatives with different assumptions
(Figure~\ref{fig:decision-tree} shows a decision diagram to clarify the choices,
and Table~\ref{tab:three-laws} gives an overview of the three laws of approximation).
Theorem~\ref{thm:fisher-comparison} gives target-Fisher
upper and lower bounds for candidates whose scores remain within a fixed range
of the target scores.  Because that candidate class is restricted, its lower
bound does not apply to the unrestricted $\mathcal E_r$.
Theorem~\ref{thm:high-mass} instead gives a lower bound on the unrestricted
$\mathcal E_r$ by focusing on a subset of tokens carrying most of the target
probability mass.

The preceding results concern approximation at finite scores.
Theorem~\ref{thm:exact-closure} treats a different effect: after softmax
saturates, a lower-rank sequence of logits can approach an attention function
that requires higher rank to realize with finite logits.  It constructs an
explicit family with a constant-factor separation between these two ranks.
Finally, Theorem~\ref{thm:multihead} extends the rank--KL bounds to a fused
update shared across attention heads, and Theorem~\ref{thm:joint-qk} extends
them to simultaneous query/key LoRA, where the score update must also satisfy
a query/key factorization constraint.

\paragraph{Contributions.}
\begin{enumerate}[leftmargin=*,itemsep=2pt]
  \item Under explicit assumptions, we characterize the best task-dependent
  attention approximation available at each LoRA rank.  Global softmax bounds
  convert centered score error into KL, and a downstream-weighted spectral
  theorem converts the remaining score error into an explicit function of
  rank.
  \item We give two alternatives when the main lower-bound constant is weak:
  target-Fisher bounds for a score-restricted candidate class, and a high-mass
  lower bound for the original unrestricted class.
  \item We prove an exact separation for an explicit Walsh family: matching its
  finite logits requires rank $k$, but after softmax saturation the same limiting
  attention can be recovered to arbitrary accuracy at rank $k-\lfloor k/3\rfloor$
  within the score space reachable by query adaptation.
  \item We extend the rank--KL analysis to a fused update shared across heads
  and to joint query/key LoRA, where a separate factorization gap measures
  whether the best effective score update can be realized by the two factors.
\end{enumerate}

\paragraph{Positioning relative to prior work.}
Prior work studies adaptive rank allocation during training
\cite{zhang2023adalora,paischer2025eva}, activation-aware approximation
\cite{yang2024corda,chen2021drone,wang2024svdllm}, finite-sample LoRA rank
selection~\cite{arunan2026sample}, and the transformations that LoRA can
express~\cite{zeng2024expressive}. Our contribution is different: for a fixed
pretrained head and a known target attention function, we bound the best
attention KL attainable at every rank, averaged over $u\sim P$, the downstream
task distribution (e.g.\ text-to-SQL examples).  The underlying
weighted approximation theorem is classical
\cite{eckart1936approximation,mirsky1960symmetric}; the new step is to connect
that task-weighted approximation problem to global upper and lower bounds on
attention KL. Section~\ref{sec:related} gives a more detailed comparison with
these and the related attention-rank and softmax-saturation literature.

\section{Problem formulation}
\label{sec:setup}

We begin with one query position in one attention head, so that every matrix
and function can be defined explicitly.  Sections~\ref{sec:multihead}
and~\ref{sec:jointqk} extend the formulation.

Let $P$ denote the downstream input distribution and draw $u\sim P$, with
$n(u)$ available token positions.
Fix pretrained logits $z_0(u)\in\R^{n(u)}$, a pretrained key matrix
$K(u)\in\R^{n(u)\times d_k}$, a query activation
$h(u)\in\R^{d_h}$, and the attention scale $\beta>0$.  A query LoRA update is
a matrix $M\in\R^{d_k\times d_h}$ with $\rank(M)\le r$.  It produces the
score vector and attention distribution
\begin{equation}
 z_M(u)=z_0(u)+\beta K(u)Mh(u),
 \qquad p_M(\cdot\mid u)=\softmax(z_M(u)).
 \label{eq:query-model}
\end{equation}
Here $h(u)$ is the query representation entering the adapted projection, and
the rows of $K(u)$ are the key vectors against which that query is scored.
$M$ is one stored matrix; the attention distribution $p_M(\cdot\mid u)$
changes with the input $u$.

Let $z_*(u)$ be fixed target logits and
$p_*(\cdot\mid u)=\softmax(z_*(u))$ the target attention function.  The target
may be arbitrary for the score-space results in Section~\ref{sec:global}.  The
spectral specialization in Theorem~\ref{thm:main-spectral} requires it to be
generated exactly by a dense query update $\Delta_*$.  A generic fully
fine-tuned Transformer can also change keys, values, earlier layers, and the
activations entering this head, so it need not satisfy that assumption.  For
such a misspecified target, Equation~\eqref{eq:functional-objective} remains
well defined, but the query-only spectral formula would need an extra
irreducible approximation term.
For rank budget $r$, define
\begin{equation}
 \boxed{\quad
 \mathcal E_r
 =\inf_{\rank(M)\le r}
   \E_{u\sim P}\KL\!\left(
     p_*(\cdot\mid u)\,\|\,p_M(\cdot\mid u)
   \right).
 \quad}
 \label{eq:functional-objective}
\end{equation}
Thus $\mathcal E_r$ is the best error attainable by the entire rank-$r$
candidate class.  An upper bound exhibits a candidate with small error.  A
lower bound applies to every candidate and certifies unavoidable error.
Inverting the bounds for a tolerance $\epsilon$ then yields sufficient and
necessary rank conditions.  Section~\ref{sec:rank-estimation} turns this
interpretation into a calibration procedure.

Softmax is invariant to adding a constant to all scores.  We therefore center
score differences with
\begin{equation}
 \PiC{n}=I_n-\frac1n\one\one^\trans,
 \qquad
 d_M(u)=\PiC{n(u)}\bigl(z_M(u)-z_*(u)\bigr).
 \label{eq:centered-error}
\end{equation}
The associated approximation objective, robust to both small and large
score errors, is
\begin{equation}
 \Psi_r
 =\inf_{\rank(M)\le r}
   \E_{u\sim P}\psi\!\left(\|d_M(u)\|_2\right).
 \label{eq:robust-objective}
\end{equation}
For the sharper upper bound, also define
\begin{equation}
 \Phi_r
 =\inf_{\rank(M)\le r}
   \E_{u\sim P}\psi_{\rm up}\!\left(\|d_M(u)\|_2\right).
 \label{eq:upper-objective}
\end{equation}

\begin{table}[H]
\centering
\small
\begin{tabularx}{\textwidth}{@{}l >{\raggedright\arraybackslash}X >{\raggedright\arraybackslash}X@{}}
\toprule
Object & Meaning & What changes? \\
\midrule
$u$ & One input from the downstream task & Varies across examples \\
$z_0(u)$ & Attention scores produced by the pretrained head & Depends on $u$ \\
$p_*(\cdot\mid u)$ & Target attention probabilities to be reproduced & Depends on $u$; the target model is fixed \\
$M$ & Candidate query update with $\rank(M)\le r$ & One fixed matrix for all inputs \\
$p_M(\cdot\mid u)$ & Attention probabilities produced with update $M$ & Depends on $u$ and on the chosen $M$ \\
$d_M(u)$ & Candidate--target score difference after removing a common shift & Depends on $u$ and on the chosen $M$ \\
$\mathcal E_r$ & Smallest average target-to-candidate KL achievable at rank $r$ & One number for each rank budget \\
\bottomrule
\end{tabularx}
\caption{Objects in the target--candidate comparison.  For each downstream
input $u$, the pretrained head and fixed target determine the behavior to be
matched; $M$ ranges over candidate updates of rank at most $r$.}
\label{tab:notation}
\end{table}
\FloatBarrier

Example~\ref{ex:text2sql} grounds this objective in a concrete downstream task.

\begin{examplebox}[label={ex:text2sql}]{Interpreting attention KL in text-to-SQL}
An input may contain the user's question, the database schema, instructions,
and previously generated SQL tokens.  The target and candidate heads each
assign attention probabilities to those token positions.  Our objective asks
how much rank is needed for the candidate to reproduce the target's attention
on such inputs.  This gives a head-level approximation measure that can be
studied alongside end-to-end measures such as SQL execution accuracy.
\end{examplebox}

\section{From score error to attention KL}
\label{sec:global}

The first result is independent of LoRA: it relates centered score error to
probability error.

\begin{theorem}[Global robust softmax bounds]
\label{thm:global-softmax}
Let $d\in\R^n$ satisfy $\one^\trans d=0$, let
$p_*=\softmax(z_*)$, and put $a=\min_i p_{*,i}>0$.  Then
\begin{equation}
 \frac{a}{2e^2}\psi(\|d\|_2)
 \le \KL\!\left(p_*\,\|\,\softmax(z_*+d)\right)
 \le \psi_{\rm up}(\|d\|_2).
 \label{eq:pointwise-robust}
\end{equation}
The upper bound does not require a positive lower bound on $a$.
\end{theorem}

The quadratic part of $\psi$ comes from local curvature of log-sum-exp.  The
linear part is unavoidable globally: a rare input can contain a very large
score error while contributing only linear or bounded KL.  A purely quadratic
global lower bound therefore cannot hold; Proposition~\ref{prop:obstruction}
gives an explicit fixed-probability-floor construction.

\begin{corollary}[Global functional bounds]
\label{cor:functional-equivalence}
If $\min_i p_{*,i}(u)\ge a>0$ almost surely, then
\begin{equation}
 \frac{a}{2e^2}\Psi_r
 \le \mathcal E_r
 \le \Phi_r.
 \label{eq:functional-equivalence}
\end{equation}
\end{corollary}

The proof of Theorem~\ref{thm:global-softmax} is short and appears in
Appendix~\ref{app:softmax}.  The rare-context obstruction to a quadratic law
is given in Appendix~\ref{app:obstruction}.

\section{Task-weighted spectral rank--KL bounds}
\label{sec:main}

Assume in this section that the target is exactly realizable by a dense query
update $\Delta_*$.  For $A=M-\Delta_*$, define
\begin{equation}
 G(u)=\beta^2K(u)^\trans\PiC{n(u)}K(u),
 \qquad
 \Sigma=\E[h(u)h(u)^\trans].
 \label{eq:key-geometry}
\end{equation}
$G(u)$ records which update directions alter centered attention scores in
context $u$; $\Sigma$ records which query directions appear on the task.

In the separable case, the random key Gram matrix $G(u)$ is independent of
the query activation $h(u)$.  Then
\begin{equation}
 \E\|\beta\PiC{n(u)}K(u)Ah(u)\|_2^2
 =\|G^{1/2}A\Sigma^{1/2}\|_\F^2,
 \qquad G=\E G(u).
 \label{eq:separable-geometry}
\end{equation}
The matrix whose spectrum matters is therefore
\begin{equation}
 D_*=G^{1/2}\Delta_*\Sigma^{1/2},
 \qquad
 T_r=\sum_{j>r}\sigma_j(D_*)^2.
 \label{eq:spectral-tail}
\end{equation}
$T_r$ is the residual squared score error after the best rank-$r$ update.
Directions annihilated by the keys or never activated by $h(u)$ disappear
automatically, a phenomenon that is illustrated in Example~\ref{ex:unused-direction}.

\begin{examplebox}[label={ex:unused-direction}]{A large update direction that the task never uses}
Suppose the target update has two singular directions.  The first has the
larger singular value, but every query activation from the downstream task is
orthogonal to it, so it never changes an attention score.  The second is
smaller but active on nearly every input.  A raw truncated SVD keeps the first
direction; the weighted spectrum above removes it and keeps the direction that
actually changes attention.  Theorem~\ref{thm:main-spectral} turns this
distinction into upper and lower error bounds.
\end{examplebox}

To turn this quadratic tail into a global robust lower bound, assume that, for
every deterministic compatible matrix $C$,
\begin{equation}
 \E\|Ch(u)\|_2^4
 \le \kappa_h\bigl(\E\|Ch(u)\|_2^2\bigr)^2,
 \qquad G(u)\preceq\Lambda G\quad\text{almost surely}.
 \label{eq:main-moment-assumptions}
\end{equation}
These uniform moment conditions prevent the quadratic error from being
carried entirely by extremely rare contexts.

\begin{theorem}[Downstream spectral rank--KL bounds]
\label{thm:main-spectral}
Suppose the target attention function is generated by a dense query update
$\Delta_*$, $\min_i p_{*,i}(u)\ge a>0$ almost surely, $G(u)$ and $h(u)$ are
independent, and the two conditions in
Equation~\eqref{eq:main-moment-assumptions} hold.  Then
\begin{equation}
 \boxed{
 \frac{a}{2e^2(1+\Lambda\sqrt{\kappa_h})}
   \psi(\sqrt{T_r})
 \le \mathcal E_r
 \le \psi_{\rm up}(\sqrt{T_r}) .}
 \label{eq:main-spectral-law}
\end{equation}
\end{theorem}

The upper bound is achieved by truncated SVD of $D_*$, which already weights
the target update by the task queries and keys.
The lower bound applies to every rank-$r$ candidate.  Both sides are quadratic
in the score scale near zero and linear at large scale, but their constants can
be far apart when the target probability floor $a$ is small.

For a desired error tolerance $\epsilon$, the upper bound gives the sufficient
rank
\begin{equation}
 r_{\rm suff}(\epsilon)
 =\min\{r:\psi_{\rm up}(\sqrt{T_r})\le\epsilon\}.
 \label{eq:sufficient-rank}
\end{equation}
Conversely, any rank attaining error at most $\epsilon$ must satisfy
\begin{equation}
 \psi(\sqrt{T_r})
 \le
 \frac{2e^2(1+\Lambda\sqrt{\kappa_h})}{a}\epsilon.
 \label{eq:necessary-rank}
\end{equation}
Inverting the error bounds in this way yields sufficient and necessary rank
conditions.

\begin{remark}[Size of the bracket]
The worst-case ratio between the upper and lower constants is at most
\begin{equation}
 \frac{2\sqrt2e^2(1+\Lambda\sqrt{\kappa_h})}{a}.
 \label{eq:main-constant-ratio}
\end{equation}
For example, it is about $2.0\times10^4$ when
$a=10^{-2}$, $\Lambda=5$, and $\kappa_h=3$.  This is a sensitivity calculation
from the theorem, not an empirical estimate.  For long, peaked attention
vectors, the probability floor can make the global lower bound numerically
weak; Section~\ref{sec:decision} gives alternatives.
\end{remark}

The proof is in Appendix~\ref{app:main-proof}.  In ordinary self-attention,
$G(u)$ and $h(u)$ are computed from the same input and need not be independent.
The following result replaces independence by direct comparison between the
true mean squared score error and the reference quadratic form used to define
$T_r$.

For $A=M-\Delta_*$, write
\begin{equation}
 Q(A)=\E\|\beta\PiC{n(u)}K(u)Ah(u)\|_2^2,
 \qquad
 S(A)=\|G^{1/2}A\Sigma^{1/2}\|_\F^2,
\end{equation}
and let $A_r^{\rm svd}$ be the weighted-SVD candidate used in the upper bound.

\begin{theorem}[Dependence-allowing spectral bounds]
\label{thm:dependent-general}
Suppose, for every feasible $A=M-\Delta_*$,
\begin{equation}
 Q(A)\ge c_rS(A),
 \qquad
 \E\|\beta\PiC{n(u)}K(u)Ah(u)\|_2^4\le\kappa_rQ(A)^2,
 \label{eq:dependent-lower}
\end{equation}
and suppose $Q(A_r^{\rm svd})\le C_r^+T_r$.  If the target probability floor
is at least $a>0$, then
\begin{equation}
 \frac{a}{2e^2(1+\sqrt{\kappa_r})}\psi(\sqrt{c_rT_r})
 \le\mathcal E_r
 \le\psi_{\rm up}(\sqrt{C_r^+T_r}).
 \label{eq:dependent-law}
\end{equation}
\end{theorem}

This theorem allows arbitrary dependence between keys and queries; its price
is that the comparison and moment constants must hold on the candidate class.
For example, if $\lambda G\preceq G(u)\preceq\Lambda G$ almost surely and the
activation moment condition in Equation~\eqref{eq:main-moment-assumptions}
holds, it applies with
$c_r=\lambda$, $C_r^+=\Lambda$, and
$\kappa_r=\kappa_h(\Lambda/\lambda)^2$.  The proof is in
Appendix~\ref{app:dependent}.

\section{Which rank--KL bound should be used?}
\label{sec:decision}

Because the minimum target probability can be
very small in long, peaked attention vectors, we provide two refinements that address
different needs.  The target-Fisher theorem gives upper and lower bounds
when score differences stay inside a fixed range.  The high-mass theorem gives a lower
bound for the original unrestricted class, using a set of tokens, chosen from
the target, that carries most probability mass.  Figure~\ref{fig:decision-tree} and
Table~\ref{tab:three-laws} make the distinction explicit:

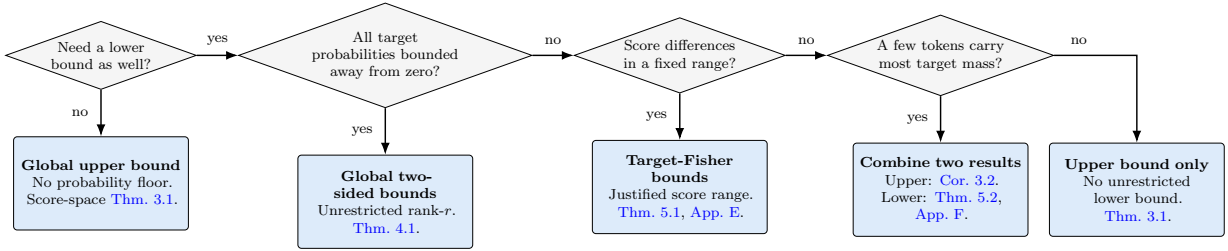
\begin{figure}[H]
\centering
\resizebox{\textwidth}{!}{%
\begin{tikzpicture}[
 font=\scriptsize,
 >=Latex,
 q/.style={draw,diamond,aspect=2.8,align=center,fill=black!4,
   minimum width=28mm,minimum height=9mm,inner sep=1pt,font=\scriptsize},
 ans/.style={draw,rounded corners=2pt,align=center,fill=TargetBlue,
   text width=32mm,minimum height=18mm,inner sep=2.5pt,font=\scriptsize},
 arr/.style={->,thick},
 edge/.style={font=\scriptsize,fill=white,inner sep=2pt,outer sep=2pt}
]
\node[q] (start) {Need a lower\\bound as well?};
\node[q,right=8mm of start] (floor) {All target\\probabilities bounded\\away from zero?};
\node[q,right=8mm of floor] (span) {Score differences\\in a fixed range?};
\node[q,right=8mm of span] (core) {A few tokens carry\\most target mass?};

\node[ans,below=9mm of start] (upper)
 {\textbf{Global upper bound}\\No probability floor.\\Score-space
  \hyperref[thm:global-softmax]{Thm.~\ref*{thm:global-softmax}}.};
\node[ans,below=9mm of floor] (global)
 {\textbf{Global two-sided bounds}\\Unrestricted rank-$r$.\\ \hyperref[thm:main-spectral]{Thm.~\ref*{thm:main-spectral}}.};
\node[ans,below=9mm of span] (fisher)
 {\textbf{Target-Fisher bounds}\\Justified score range.\\ \hyperref[thm:fisher-comparison]{Thm.~\ref*{thm:fisher-comparison}},
  \hyperref[app:fisher]{App.~\ref*{app:fisher}}.};
\node[ans,below=9mm of core] (highmass)
 {\textbf{Combine two results}\\Upper: \hyperref[cor:functional-equivalence]{Cor.~\ref*{cor:functional-equivalence}}.\\
  Lower: \hyperref[thm:high-mass]{Thm.~\ref*{thm:high-mass}},
  \hyperref[app:highmass]{App.~\ref*{app:highmass}}.};
\node[ans,right=4mm of highmass] (nolower)
 {\textbf{Upper bound only}\\No unrestricted lower bound.\\ \hyperref[thm:global-softmax]{Thm.~\ref*{thm:global-softmax}}.};

\draw[arr] (start.east)--node[edge,above=2pt]{yes}(floor.west);
\draw[arr] (start.south)--node[edge,left=3pt]{no}(upper.north);
\draw[arr] (floor.east)--node[edge,above=2pt]{no}(span.west);
\draw[arr] (floor.south)--node[edge,left=3pt]{yes}(global.north);
\draw[arr] (span.east)--node[edge,above=2pt]{no}(core.west);
\draw[arr] (span.south)--node[edge,left=3pt]{yes}(fisher.north);
\draw[arr] (core.south)--node[edge,left=3pt]{yes}(highmass.north);
\draw[arr] (core.east) -| node[edge,pos=.15,above=2pt]{no} (nolower.north);
\end{tikzpicture}}\par
\caption{\textbf{Which rank--KL theorem to use}.
Start at the left and stop at the first result whose assumptions you can justify.
An upper bound shows that some rank-\(r\) update achieves at most the stated error;
a lower bound shows that every rank-\(r\) update incurs at least the stated error.
The Fisher bounds apply only to candidates whose scores remain within a fixed
range of the target scores. Consequently, their lower bound does not apply to
the unrestricted optimum \(\mathcal E_r\).}
\label{fig:decision-tree}
\end{figure}

\begin{table}[ht]
\centering
\small
\begin{tabularx}{\textwidth}{@{}l c c >{\raggedright\arraybackslash}X >{\raggedright\arraybackslash}X@{}}
\toprule
Route & Upper bound & Lower bound & Candidate class & Main price \\
\midrule
Global & \cmark & \cmark & unrestricted rank-$r$ & minimum target probability and moment assumptions on keys/queries \\
Target Fisher & \cmark & \cmark & rank-$r$ with a justified range of score differences & span-dependent constants and the target Fisher inner product \\
High mass & \xmark & \cmark & unrestricted rank-$r$ & tokens carrying most target mass, a conditional probability floor, and the key/query Gram on that set \\
\bottomrule
\end{tabularx}
\caption{The three rank--KL routes and their logical scope.}
\label{tab:three-laws}
\end{table}

\subsection{Target-Fisher bounds under a controlled score span}

Let $H_*=\diag(p_*)-p_*p_*^\trans$ and let
$R(d)=\max_i d_i-\min_i d_i$.  Define
\begin{equation}
 c_-(R)=\frac{R-1+e^{-R}}{R^2},
 \qquad
 c_+(R)=\frac{e^R-1-R}{R^2},
 \label{eq:fisher-constants}
\end{equation}
with both values set to $1/2$ at $R=0$.

\begin{theorem}[Target-Fisher score bounds]
\label{thm:fisher-comparison}
For every score displacement $d$,
\begin{equation}
 c_-(R(d))d^\trans H_*d
 \le \KL(p_*\|\softmax(z_*+d))
 \le c_+(R(d))d^\trans H_*d.
 \label{eq:fisher-comparison}
\end{equation}
Consequently, on the candidate class whose score differences from the target
have range at most $R_0$, the optimal KL is bounded above and below by
$c_+(R_0)$ and $c_-(R_0)$ times the corresponding target-Fisher quadratic
optimum.
\end{theorem}

There is no explicit minimum-probability constant; small target
probabilities enter through the Fisher matrix $H_*$.  More importantly, the
optimized class is restricted by the span condition.  Because that class is
a subset of the unrestricted rank-$r$ class, its lower bound cannot be
transferred to $\mathcal E_r$.  The spectral statements, including the extra
span condition needed for a matching upper bound, and the proofs are in
Appendix~\ref{app:fisher}.

\subsection{An unrestricted high-mass lower bound}

For each context, let $S(u)$ be a measurable token set chosen from the target
before optimizing the candidate.  Suppose it carries at least $1-\delta$ of
the target mass and the conditional target distribution on $S(u)$ has minimum
probability at least $a_{\core}$.

\begin{theorem}[High-mass lower bound]
\label{thm:high-mass}
Under the preceding assumptions,
\begin{equation}
 \mathcal E_r
 \ge \frac{(1-\delta)a_{\core}}{2e^2}
 \inf_{\rank(M)\le r}
 \E\psi\!\left(
   \left\|\PiC{S(u)}(z_M(u)-z_*(u))_{S(u)}\right\|_2
 \right).
 \label{eq:high-mass-robust}
\end{equation}
\end{theorem}

This is a lower bound on the original unrestricted $\mathcal E_r$.  Under
additional control of the key/query Gram and moments on $S(u)$, the
right-hand side reduces to the same spectral tail, but only on those tokens.
It does not give a full-KL upper bound:
matching conditional scores inside $S(u)$ need not match the total mass
assigned to $S(u)$.  The proof and spectral forms are in
Appendix~\ref{app:highmass}.

\section{Saturation changes the relevant notion of rank}
\label{sec:saturation}

The main theorem describes approximation of finite logits, but softmax also has a
boundary regime in which logits can diverge while their distributions
converge. Example~\ref{ex:saturation} shows the basic saturation mechanism before the
rank separation is stated formally.

\begin{examplebox}[label={ex:saturation}]{Different logits, the same saturated attention}
For two tokens,
\begin{equation*}
 \softmax(T,0)\longrightarrow(1,0),
 \qquad
 \softmax(2T,0)\longrightarrow(1,0)
 \quad\text{as }T\to\infty.
\end{equation*}
The difference between the two logit vectors grows with $T$, but their
attention distributions converge to the same limit.  The results below prove
a stronger effect: for explicit families of targets, saturation reduces the
rank needed to approximate the entire attention function.
\end{examplebox}

For a boundary target $P_\infty$ and an allowed centered score space $L$, define
its \emph{$L$-relative softmax closure rank} $r_{\rm cl}^{L}(P_\infty)$ as the
smallest $r$ for which there is a sequence of score matrices of rank at most
$r$, with every score column in $L$, whose columnwise softmax distributions
converge to $P_\infty$.  For query adaptation, $L$ is the score space reachable
through the fixed key matrix, so the restriction is part of the LoRA model.
This rank can be smaller than the rank required to match every finite target
logit exactly.

\begin{theorem}[Exact closure rank of an isolated-triple Walsh family]
\label{thm:exact-closure}
For every $k\ge3$, there is a query-only Walsh-attention family with $k$
contexts and fewer than $4k^2+8$ token positions such that every exact
realization with finite logits has update rank $k$.  For the limiting target
$P_\infty^{(k)}$ of this family and its Walsh score space $L_k$, the exact
relative closure rank is
\begin{equation}
 r_{\mathrm{cl}}^{L_k}(P_\infty^{(k)})
 =k-\lfloor k/3\rfloor.
 \label{eq:closure-rank}
\end{equation}
This equality is for the isolated-triple family constructed in
Appendix~\ref{app:saturation}; other target families can have smaller closure
rank, as the construction below demonstrates.
\end{theorem}

The construction groups Walsh characters into isolated triples.  A rank-two
path per triple gives the upper bound.  For the lower bound, restricted
Fourier identities and columnwise normalization force every limiting triple
block to retain rank at least two.  Saturation therefore yields a genuine
constant-factor reduction, but in this family it cannot collapse the
rank of the attention function to $o(k)$.

A separate linear-token family gives a stronger achieved ratio: with fewer
than $8(k+1)$ token positions, rank at most
\begin{equation}
 k-\max\{\lfloor k/3\rfloor,3\lfloor k/7\rfloor\}
 \label{eq:linear-family-rank}
\end{equation}
attains vanishing KL, reaching ratio $4/7$ on complete seven-context blocks.
That construction concerns a different target family and does not identify
its minimum closure rank.  The two results and their proofs are kept separate in
Appendix~\ref{app:saturation}.

\section{Fused multi-head LoRA}
\label{sec:multihead}

Standard query LoRA for multi-head attention often constrains one fused
projection matrix, vertically divided into head blocks.  This is not the same
as assigning an independent rank to every head: one fused rank direction can
serve several heads.

For $j=1,\ldots,H$, let $K_j(u)\in\R^{n_j(u)\times d_{k,j}}$ and let all heads
share the query activation $h(u)\in\R^{d_h}$.  Stack the head updates as
$M=(M_1^\trans,\ldots,M_H^\trans)^\trans$ and impose the single fused
constraint $\rank(M)\le r$.  Assume the target is generated by the similarly
stacked dense update $\Delta_*^{\MH}$.  Define
\begin{align}
 d_{M,j}(u)&=\beta_j\PiC{n_j(u)}K_j(u)(M_j-\Delta_{*,j})h(u),\\
 \mathcal E_r^{\MH}
 &=\inf_{\rank(M)\le r}\E\sum_{j=1}^H
   \KL(p_{*,j}(\cdot\mid u)\|p_{M,j}(\cdot\mid u)),\\
 G_{\MH}(u)&=\operatorname{blockdiag}\!\left(
   \beta_j^2K_j(u)^\trans\PiC{n_j(u)}K_j(u):j=1,\ldots,H\right),
 \qquad G_{\MH}=\E G_{\MH}(u),\\
 T_r^{\MH}&=\sum_{j>r}\sigma_j\!\left(
 G_{\MH}^{1/2}\Delta_*^{\MH}\Sigma^{1/2}\right)^2,
 \qquad \Sigma=\E[h(u)h(u)^\trans].
 \label{eq:multihead-definitions}
\end{align}

\begin{theorem}[Fused multi-head rank--KL bounds]
\label{thm:multihead}
Suppose every target probability is at least $a_{\min}>0$ almost surely,
$G_{\MH}(u)$ is independent of $h(u)$,
$G_{\MH}(u)\preceq\Lambda_{\MH}G_{\MH}$ almost surely, and
\begin{equation}
 \E\|Ch(u)\|_2^4
 \le\kappa_h\bigl(\E\|Ch(u)\|_2^2\bigr)^2
 \quad\text{for every compatible deterministic }C.
 \label{eq:multihead-moment}
\end{equation}
Then
\begin{equation}
 \frac{a_{\min}}{2e^2(1+\Lambda_{\MH}\sqrt{\kappa_h})}
 \psi(\sqrt{T_r^{\MH}})
 \le\mathcal E_r^{\MH}
 \le\min\{T_r^{\MH}/4,\sqrt{2H T_r^{\MH}}\}.
 \label{eq:multihead-law}
\end{equation}
\end{theorem}

The upper/lower bracket can widen by a factor of order $\sqrt H$ because the
objective sums KL over heads; the aggregation inequality producing this
factor is tight without further assumptions on how error is distributed
across heads.
If each head instead has its own adapter and integer rank $r_h$, the feasible
set separates after the ranks are fixed; allocating a total budget $R$ is the
discrete problem $\min_{\sum_h r_h\le R}\sum_h\mathcal E_{h,r_h}$.
The proof and the underlying pointwise aggregation inequality are in
Appendix~\ref{app:multihead}.

\section{Joint query/key LoRA}
\label{sec:jointqk}

Practical LoRA may update both queries and keys.  Let
$Q_0,K_0\in\R^{p\times d}$ be pretrained factors and let
$A,B\in\R^{p\times d}$ be query and key updates with ranks at most
$r_Q,r_K$.  On ordinary dot-product scores they change the scores by the
effective update
\begin{equation}
 C(A,B)
 =(K_0+B)^\trans(Q_0+A)-K_0^\trans Q_0
 =K_0^\trans A+B^\trans(Q_0+A).
 \label{eq:effective-qk}
\end{equation}
The bound $\rank(C(A,B))\le r_Q+r_K$ is known~\cite{zeng2024expressive}:
grouping the bilinear term gives the sharp containment
\begin{equation}
 \rank(C(A,B))\le r_Q+r_K.
 \label{eq:qk-rank}
\end{equation}

Let a context provide token representations
$X(u)\in\R^{n(u)\times d}$ and a query representation $h(u)\in\R^d$.  For
target factors $(A_*,B_*)$, put $C_*=C(A_*,B_*)$ and define the centered score
error and actual factor-class objective
\begin{align}
 d_{A,B}(u)
 &=\beta\PiC{n(u)}X(u)[C(A,B)-C_*]h(u),\\
 \mathcal E^{\QK}_{r_Q,r_K}
 &=\inf_{\substack{\rank(A)\le r_Q\\\rank(B)\le r_K}}
 \E\KL(p_*\|p_{A,B}).
 \label{eq:qk-objective}
\end{align}
Let
\begin{equation}
 G_X(u)=\beta^2X(u)^\trans\PiC{n(u)}X(u),
 \qquad G_X=\E G_X(u),
 \qquad \Sigma=\E[h(u)h(u)^\trans].
 \label{eq:qk-geometry}
\end{equation}
If $G_X(u)$ and $h(u)$ are independent, the mean squared score error equals
$\|G_X^{1/2}[C(A,B)-C_*]\Sigma^{1/2}\|_\F^2$.  Define
\begin{equation}
 D_*^{\eff}=G_X^{1/2}C_*\Sigma^{1/2},
 \qquad
 T_s^{\eff}=\sum_{j>s}\sigma_j(D_*^{\eff})^2,
 \qquad s=r_Q+r_K.
 \label{eq:qk-tail}
\end{equation}
Let $\mathcal S_s^{\rm svd}$ be the set of rank-$s$ matrices attaining the
weighted approximation error $T_s^{\eff}$.  An element of this set need not be
writable as query and key updates of size $p\times d$ with the separate rank
budgets.  Define the extra error from that factorization gap by
\begin{equation}
 \rho_{r_Q,r_K}
 =\inf_{C_s\in\mathcal S_s^{\rm svd}}
   \inf_{\substack{\rank(A)\le r_Q\\\rank(B)\le r_K}}
 \|G_X^{1/2}[C(A,B)-C_s]\Sigma^{1/2}\|_\F.
 \label{eq:qk-rho}
\end{equation}

\begin{theorem}[Joint Q/K rank--KL bounds]
\label{thm:joint-qk}
Suppose $\min_i p_{*,i}(u)\ge a>0$ almost surely, $G_X(u)$ is independent of
$h(u)$ so that the mean-squared identity above holds, and, for every feasible pair $(A,B)$,
\begin{equation}
 \E\|d_{A,B}(u)\|_2^4
 \le\kappa\bigl(\E\|d_{A,B}(u)\|_2^2\bigr)^2.
 \label{eq:qk-moment}
\end{equation}
Then
\begin{equation}
 \frac{a}{2e^2(1+\sqrt\kappa)}
   \psi(\sqrt{T_s^{\eff}})
 \le \mathcal E^{\QK}_{r_Q,r_K}
 \le \psi_{\rm up}\!\left(\sqrt{T_s^{\eff}}+\rho_{r_Q,r_K}\right).
 \label{eq:joint-qk-law}
\end{equation}
\end{theorem}

The lower bound follows from containment in the effective rank-$s$ class.  The
reverse containment generally fails: every updated width-$p$ head also
satisfies $\rank(K_0^\trans Q_0+C)\le p$.  Exact one-sided support
conditions or a sequential two-sided factorization make
$\rho_{r_Q,r_K}=0$; otherwise that extra error can be large.  The realization lemmas,
a constructive target-factor upper bound, and a fused multi-head lower bound
are proved in Appendix~\ref{app:jointqk}.

The two sides form a useful bracket only when
$\rho_{r_Q,r_K}$ is comparable to, or smaller than,
$\sqrt{T_s^{\eff}}$.  In general, computing $\rho_{r_Q,r_K}$ is a nonconvex
factorization problem over the weighted rank-$s$ optimizers; the theorem does
not provide a tractable procedure or a universal upper bound for it.  The
cases with $\rho_{r_Q,r_K}=0$ therefore identify the cleanest regime for this
extension.

RoPE inserts a relative rotation between the query and key factors.  Every
relative-position slice still has rank at most $r_Q+r_K$, so the global robust
bound remains valid for the assembled scores.  The candidate is then a coupled
family of effective matrices, however, and the single-matrix spectral tail in
Theorem~\ref{thm:joint-qk} does not apply automatically.

This RoPE-specific obstruction is not universal.  Architectures that omit
explicit positional embeddings and use NoPE, such as Kimi
K3~\cite{kimit2026k3}, keep token order only through causal computation.  For
those heads, the ordinary dot-product formulation is the more direct starting
point, subject to the model's other attention mechanisms.

\section{Estimating LoRA rank from downstream data}
\label{sec:rank-estimation}

The theory can be used once a dense or high-rank target update and a calibration
set from the downstream task are available.  Its output is a pair of error
curves over rank: an upper curve achieved by a constructed candidate and a
lower curve that no candidate in the stated class can beat.  These curves may
identify one rank, or they may leave an interval that the theory does not
resolve.

For query-only adaptation, the procedure is:
\begin{enumerate}[leftmargin=*,itemsep=3pt]
  \item \textbf{Check the target class.} Verify that the target attention is
  generated by a query update $\Delta_*$.  If both queries and keys change,
  use Section~\ref{sec:jointqk}; a target produced by unrestricted fine-tuning
  may also contain an irreducible query-only approximation error.
  \item \textbf{Collect calibration inputs.} Sample complete inputs from the
  downstream task and run the fixed pretrained and target heads.  Store the
  query activations $h(u)$, centered key Gram matrices $G(u)$, and target
  attention probabilities.  Complete inputs, rather than individual token
  positions treated as independent observations, are the sampling units.
  \item \textbf{Choose the applicable bound.} Use
  Figure~\ref{fig:decision-tree}.  If the separability assumptions of
  Theorem~\ref{thm:main-spectral} are not credible, use the dependence-allowing
  geometry in Appendix~\ref{app:dependent}.  If the target probability floor
  is too small, evaluate the target-Fisher or high-mass route instead.
  \item \textbf{Estimate the spectral tail.} For the main route, estimate
  $G=\E G(u)$ and $\Sigma=\E[h(u)h(u)^\trans]$, form
  $\widehat D_*=\widehat G^{1/2}\Delta_*\widehat\Sigma^{1/2}$, and compute
  $\widehat T_r=\sum_{j>r}\sigma_j(\widehat D_*)^2$ for every rank of
  interest.  The truncated SVD of $\widehat D_*$ gives the corresponding
  upper-bound candidate in the weighted coordinates.  On the supported
  subspaces, the corresponding query update is
  $\widehat M_r=\widehat G^{\dagger/2}(\widehat D_*)_r
  \widehat\Sigma^{\dagger/2}$, where $(\widehat D_*)_r$ is the rank-$r$
  truncated SVD.
  \item \textbf{Supply and stress-test the constants.} The probability floor,
  almost-sure geometry bound, and uniform moment constant are assumptions of
  the population theorem, not quantities certified by a finite calibration
  set.  Use analytic bounds when available, or report $L_r$ over a range of
  assumed values.  Sample estimates are diagnostics, not finite-sample
  certificates.  The upper curve $U_r$ comes from the constructed candidate
  and does not require these lower-bound constants.
  \item \textbf{Compare with the desired tolerance.} If $U_r\le\epsilon$,
  the upper bound exhibits a rank-$r$ candidate within tolerance.  If
  $L_r>\epsilon$, no rank-$r$ candidate in the relevant class can meet the
  tolerance.  The smallest rank certified by the upper curve is sufficient;
  ranks rejected by the lower curve are impossible under the theorem's
  assumptions.  Any gap between the two remains unresolved.
\end{enumerate}

\begin{examplebox}[label={ex:rank-selection}]{Estimating rank on a given calibration set}
Suppose a high-rank query adapter has already been trained, for instance for a text-to-SQL
task as in Example~\ref{ex:text2sql}.  A
calibration set contains complete prompts with the question, schema,
instructions, and SQL prefix available at each attention call.  The upper and
conditional lower curves divide the tested ranks into three sets: ranks whose
constructed candidate meets the KL tolerance, ranks excluded by the
conditional lower bound, and ranks for which the two bounds do not decide.
The width of the unresolved set depends on the gap between the upper and lower
constants and should be reported rather than replaced by a single selected
rank.
\end{examplebox}

The population theorems do not provide confidence intervals for these plug-in
estimates.  A held-out calibration split can test whether the predicted
rank--KL curve tracks measured target-to-candidate KL, but it does not turn the
current bounds into a finite-sample guarantee.  The relevant comparison is
against the raw spectrum of $\Delta_*$ and activation-only rank criteria: the
question is whether task-weighted tails predict attainable attention KL more
accurately, not whether a particular optimizer produces a U-shaped validation
curve.

\section{Related work}
\label{sec:related}

\paragraph{LoRA and expressivity.}
LoRA introduced low-rank updates as a parameter-efficient adaptation
mechanism~\cite{hu2022lora}.  Zeng and Lee~\cite{zeng2024expressive} analyze
LoRA expressivity, including products of adapted matrices and exact
Transformer score matching.  We use their bound
$\rank(C)\le r_Q+r_K$ in the joint Q/K analysis and then quantify the
additional error caused when the best effective score update cannot be
realized by factors with the separate rank budgets.  Duranthon et
al.~\cite{duranthon2026highdimensional} give a statistical theory of rank-one
LoRA fine-tuning in a solvable attention model; their object is learning-curve
asymptotics rather than error as a function of rank for a prescribed target.
Arunan~\cite{arunan2026sample} studies finite-sample estimation and rank
selection for empirical risk minimization over rank-constrained LoRA.  In a
well-specified locally quadratic model, that work derives matching
$\widetilde\Theta(rd/n)$ estimation rates and an under-ranking bias governed by
the raw singular values of the target update.  That work studies statistical
estimation from finite training data.  We hold the target fixed and study the
approximation error remaining at each rank, with a spectrum weighted by the
task queries and keys and a global error scale determined by softmax.

\paragraph{Weighted approximation.}
The weighted low-rank approximation theorem is a classic result.
\cite{eckart1936approximation,mirsky1960symmetric}.  Methods such as DRONE and
SVD-LLM use related data-aware decompositions for model compression
\cite{chen2021drone,wang2024svdllm}.  CorDA orients a weight decomposition by
downstream activation covariance, while EVA initializes and reallocates LoRA
rank using activation variance~\cite{yang2024corda,paischer2025eva}.  These
methods motivate task-aware approximation directions.  Our analysis places
the weighted approximation inside the target-to-candidate attention KL and
derives both upper and lower bounds, including the global scale $\psi$ needed
for small and large score errors.

\paragraph{Task-intrinsic attention rank and low-rank logits.}
Yoon~\cite{yoon2026entropic} defines attention-native intrinsic rank as the
minimum query--key kernel rank realizing a task and studies controlled
deficiency and recovery phenomena.  We study the narrower class of LoRA
updates of a pretrained head, and a prescribed attention target under KL.  Golowich et
al.~\cite{golowich2026sequences} study low-rank approximation of extended
language-model logit matrices under average KL, without pretrained LoRA
attention or separate Q/K update constraints.

\paragraph{Attention rank and saturation.}
Attention head dimension and score-matrix rank have been studied as
expressivity constraints~\cite{bhojanapalli2020lowrank}.  Sparse-target
cross-entropy and diverging-margin geometry are adjacent to our saturation
branch~\cite{zhao2024implicit}; low pre-softmax rank can also produce high
post-softmax rank~\cite{masarczyk2025unpacking}.  Boundary supports of fixed
discrete exponential families are classical~\cite{csiszar2005closures,rauh2011support}.
Sign rank and rounding rank concern finite, exact realization of prescribed
sign or threshold patterns
\cite{paturi1986communication,neumann2016rounding}, whereas our closure rank
allows a sequence of scores to diverge while its softmax converges.  Our lower
bound uses restricted Fourier identities, columnwise normalization, and the
closedness of a bounded-rank matrix set.  We do not know a direct implication
in either direction between the cited Walsh sign-rank bounds and the
$L$-relative closure-rank formula in Theorem~\ref{thm:exact-closure}.
Low-rank softmax models have also been studied through
argmaxability~\cite{grivas2022unargmaxable,grivas2023sigmoid}.
Basri and Jacobs~\cite{basri2026softmax} show that low-dimensional softmax
embeddings can preserve ratios among selected top-token probabilities while
off-support mass vanishes.  These results establish the broad
support-separation and diverging-margin mechanisms.  Our saturation theorem
makes a different comparison within one explicit multi-context family: the
rank required to match finite logits versus the exact softmax closure rank.

\paragraph{Softmax curvature.}
The Hessian comparison behind our target-Fisher bounds is related to
generalized self-concordant analyses of logistic and softmax losses
\cite{bach2010selfconcordant,marteau2019newton}.  We specialize this curvature
principle through a finite-dimensional Bregman calculation for KL from the
target attention to the candidate.

\section{Limitations}

The theory starts from an available target update.  It can assess compression
of that target or the rank needed to reproduce it, but it does not prescribe
a rank before any target behavior is known.

The bounds are population statements.  From a calibration set, one can plug in
estimates of the task geometry ($G$, $\Sigma$) and the corresponding spectral
tails, although we do not provide finite-sample confidence intervals for the
resulting rank--KL curves.  Other inputs to the lower bounds---the probability
floor, the almost-sure geometry constant, and the uniform moment
constant---enter as theorem assumptions rather than as quantities certified by
a finite sample, so a practical lower curve must rely on analytic bounds or
report sensitivity to those values.  The spectral specialization likewise
assumes that the target attention is realizable within the stated query-only
or joint query/key adapter class; unrestricted fine-tuning may leave an
additional approximation error outside that class.

The results describe the best candidate in a rank-constrained class, not the
path taken by a training algorithm.  They therefore separate representational
capacity from optimization but do not guarantee that SGD or another optimizer
finds the candidate attaining the upper bound.  The constants can also be
loose for long, sharply peaked attention vectors.  The target-Fisher and
high-mass results offer alternatives in that regime, but their own assumptions
must still be checked.  We have not established that any of these constants
are numerically tight on trained, real-model attention heads (for more information on
computational checks, see Appendix~\ref{app:computational-checks}).

Finally, the guarantees concern attention probabilities.  Converting them
into bounds on layer outputs or task loss requires assumptions about values,
output projections, residual paths, feed-forward networks, and later layers.
The saturation families isolate a genuine nonlinear effect but are explicit
constructions rather than models of typical language data.  For joint query/key
LoRA, RoPE couples the factors across relative positions, so the single-matrix
spectral specialization does not apply directly, although it is worth noting
that this particular obstruction
is absent in NoPE architectures such as Kimi K3~\cite{kimit2026k3}.

\section{Conclusion}

We studied how much LoRA rank is needed to reproduce a known attention function
on inputs from a downstream task.  The main message is that this question has
a task-dependent answer: it is not settled by the raw rank or singular values
of the target weight update alone, but by the part of that update that the
task's queries and keys actually activate.  Under explicit assumptions, we
connect that activated component to the best expected attention KL attainable
at each rank.  A global softmax comparison turns centered score error into KL,
and a downstream-weighted spectral theorem turns the remaining score error
into an explicit function of rank through the scale
$\psi(t)=\min\{t^2,t\}$.  Together, these steps replace rank heuristics and
post-hoc sweeps with a two-sided, rank-indexed bracket on representational
error at the attention layer itself.

The same framework extends when the standard constants are weak or the adapter
is richer than query-only LoRA.  Target-Fisher and high-mass lower bounds
provide alternative routes when the global probability-floor assumption is
uninformative, and an explicit saturation family shows that softmax can
separate the rank needed to match finite logits from the rank needed to match
the limiting attention function.  Fused multi-head analysis accounts for a
shared rank budget across heads, and joint query/key LoRA isolates the extra
error from factorizing an effective score update into separate low-rank
factors.

Given a target adapter and a calibration set drawn from the task, the theory
therefore yields upper and lower error curves over rank and makes explicit
where the available assumptions certify a rank, exclude it, or leave the
decision unresolved.

\bibliographystyle{plain}
\bibliography{references}

\clearpage
\appendix
\section*{Guide to the appendix}
Appendices~\ref{app:softmax}--\ref{app:dependent} prove the global softmax,
moment, spectral, and dependence-allowing bounds.  Appendices~\ref{app:fisher}
and~\ref{app:highmass} contain the two alternative lower-bound routes.
Appendix~\ref{app:saturation} gives the saturation constructions and closure-rank
lower bound, and Appendices~\ref{app:multihead}--\ref{app:jointqk} prove the
multi-head and joint query/key extensions.

\paragraph{Computational checks.}
\label{app:computational-checks}
The accompanying code is available at
\url{https://github.com/gerardpc/lora-rank-project}
and verifies the displayed finite constructions, rank calculations, and
selected inequalities on small instances.  These checks are provided for
reproducibility and debugging and are not used in the proofs.

\section*{Appendix contents}
\begin{enumerate}[label=\Alph*.,leftmargin=2em,itemsep=1pt,topsep=4pt]
  \item \hyperref[app:softmax]{Proof of the global softmax bounds}
  \item \hyperref[app:obstruction]{Why a purely quadratic global lower bound is impossible}
  \item \hyperref[app:main-proof]{Moment and spectral lemmas}
  \item \hyperref[app:dependent]{Dependence-allowing rank--KL bounds}
  \item \hyperref[app:fisher]{Target-Fisher bounds}
  \item \hyperref[app:highmass]{High-mass lower bounds}
  \item \hyperref[app:saturation]{Saturation constructions and lower bound}
  \item \hyperref[app:multihead]{Fused multi-head proofs}
  \item \hyperref[app:jointqk]{Joint query/key proofs}
\end{enumerate}
\bigskip
\input{appendix}

\end{document}

%% file: appendix.tex
\section{Proof of the global softmax bounds}
\label{app:softmax}

For $z\in\R^n$, write
$\operatorname{lse}(z)=\log\sum_i e^{z_i}$ and
$H(z)=\diag(p(z))-p(z)p(z)^\trans$, where
$p(z)=\softmax(z)$.

\begin{proof}[Proof of Theorem~\ref{thm:global-softmax}]
The KL divergence is the log-sum-exp Bregman divergence
\begin{equation}
 F(d)=\operatorname{lse}(z_*+d)-\operatorname{lse}(z_*)
      -\langle p_*,d\rangle.
 \label{eq:app-bregman}
\end{equation}
Taylor's formula with integral remainder gives
\begin{equation}
 F(d)=\int_0^1(1-t)d^\trans H(z_*+td)d\,dt.
 \label{eq:app-bregman-int}
\end{equation}
Since $\|H(z)\|_{\op}\le1/2$,
$F(d)\le\|d\|_2^2/4$.  Also both
$\operatorname{lse}(z_*+d)-\operatorname{lse}(z_*)$ and
$\langle p_*,d\rangle$ lie in $[\min_i d_i,\max_i d_i]$.  Therefore
\begin{equation}
 F(d)\le \max_i d_i-\min_i d_i\le\sqrt2\|d\|_2.
\end{equation}
Taking the smaller of the quadratic and linear bounds proves
$F(d)\le\psi_{\rm up}(\|d\|_2)$.

For the lower bound, first suppose $\|d\|_2\le1$.  For every $t\in[0,1]$,
\begin{equation}
 p_i(z_*+td)
 =\frac{p_{*,i}e^{td_i}}{\sum_jp_{*,j}e^{td_j}}
 \ge ae^{-2}.
 \label{eq:app-floor-segment}
\end{equation}
For any probability vector $q$ satisfying $\min_iq_i\ge\alpha$ and any
centered vector $v$,
\begin{align}
 v^\trans(\diag(q)-qq^\trans)v
 &=\sum_iq_i(v_i-\mu_q)^2 \\
 &\ge\alpha\sum_i(v_i-\mu_q)^2
 \ge\alpha\|v\|_2^2,
\end{align}
where $\mu_q=\sum_iq_iv_i$ and the last inequality uses that zero is the
Euclidean-optimal centering constant for a centered $v$.  Inserting
$\alpha=ae^{-2}$ into~\eqref{eq:app-bregman-int} yields
\begin{equation}
 F(d)\ge\frac{ae^{-2}}2\|d\|_2^2,
 \qquad \|d\|_2\le1.
 \label{eq:app-local-lower}
\end{equation}

If $\|d\|_2=s\ge1$, put $v=d/s$ and $g(t)=F(tv)$.  Convexity and $g(0)=0$
imply that $g(t)/t$ is nondecreasing.  Hence
\begin{equation}
 F(d)=g(s)\ge s g(1)\ge\frac{ae^{-2}}2s.
\end{equation}
Combining this with~\eqref{eq:app-local-lower} proves the lower bound.
\end{proof}

\begin{proof}[Proof of Corollary~\ref{cor:functional-equivalence}]
Apply Theorem~\ref{thm:global-softmax} pointwise to the centered displacement
$d_M(u)$, take expectations, and then take the infimum over the same
rank-constrained candidate class on both sides.
\end{proof}

\section{Why a purely quadratic global lower bound is impossible}
\label{app:obstruction}

\begin{proposition}[Fixed-floor rare-context obstruction]
\label{prop:obstruction}
There is a family with target probability floor $1/3$ and mean key Gram
$G=I$ for which the best rank-one quadratic error tends to one while
the best rank-one expected attention KL tends to zero.
\end{proposition}

\begin{proof}
Let $d=2$, $n=3$, and choose $K\in\R^{3\times2}$ with orthonormal columns in
$\one^\perp$.  Draw
\begin{equation}
 h=e_1\quad\text{with probability }1-p,
 \qquad
 h=Te_2\quad\text{with probability }p,
 \qquad T=\sqrt{(1-p)/p}.
\end{equation}
Take the pretrained query map to be $-I_2$ and the dense target update to be
$I_2$.  The target query map is therefore zero, so the target attention is
uniform and its floor is $1/3$.  The activation covariance is
\begin{equation}
 \Sigma=(1-p)e_1e_1^\trans+pT^2e_2e_2^\trans=(1-p)I_2.
\end{equation}
The best quadratic rank-one approximation to $I_2$ has error $1-p\to1$.

Choose the rank-one candidate $M=e_1e_1^\trans$.  It is exact when $h=e_1$.
When $h=Te_2$, its centered scores are $-Tq_2$, where $q_2$ is a centered unit
column of $K$.  Thus
\begin{equation}
 \KL(\mathrm{unif}_3\|\softmax(-Tq_2))
 =\operatorname{lse}(-Tq_2)-\log3\le T.
\end{equation}
The population KL is at most $pT=\sqrt{p(1-p)}\to0$.
\end{proof}

\section{Moment and spectral lemmas}
\label{app:main-proof}

For a deterministic error matrix $A$, define
\begin{equation}
 X_A(u)=\|\beta\PiC{n(u)}K(u)Ah(u)\|_2,
 \qquad Q(A)=\E X_A(u)^2.
 \label{eq:app-XA}
\end{equation}

\begin{lemma}[Fourth-moment bridge]
\label{lem:moment-bridge}
If a nonnegative random variable $X$ satisfies
$\E X^4\le\kappa(\E X^2)^2$, then, with $\mu=\E X^2$,
\begin{equation}
 \E\psi(X)\ge\frac{\psi(\sqrt\mu)}{1+\sqrt\kappa}.
 \label{eq:app-moment-bridge}
\end{equation}
\end{lemma}

\begin{proof}
For $x\ge0$, $\psi(x)\ge x^2/(1+x)$.  Cauchy--Schwarz gives
\begin{equation}
 \mu^2
 \le \E\!\left[\frac{X^2}{1+X}\right]
       \E[X^2(1+X)].
\end{equation}
Moreover,
$\E X^3\le\sqrt{\E X^2\E X^4}\le\sqrt\kappa\mu^{3/2}$.  Therefore
\begin{equation}
 \E\psi(X)
 \ge\frac{\mu}{1+\sqrt{\kappa\mu}}
 \ge\frac{\psi(\sqrt\mu)}{1+\sqrt\kappa}.
\end{equation}
For the last inequality: if $\mu\le1$, divide by $\mu$ and use
$\sqrt\mu\le1$; if $\mu\ge1$, divide by $\sqrt\mu$ and use
$1/\sqrt\mu\le1$.
\end{proof}

\begin{lemma}[Uniform moment condition]
\label{lem:uniform-moment}
Assume $G(u)$ and $h(u)$ are independent, that for every deterministic
compatible matrix $C$,
\begin{equation}
 \E\|Ch\|_2^4\le\kappa_h(\E\|Ch\|_2^2)^2,
 \label{eq:app-hyper}
\end{equation}
and that $G(u)\preceq\Lambda G$ almost surely, where $G=\E G(u)$.  Then every
deterministic $A$ satisfies
\begin{equation}
 \E X_A^4\le\kappa_h\Lambda^2(\E X_A^2)^2.
 \label{eq:app-uniform-moment}
\end{equation}
\end{lemma}

\begin{proof}
Condition on $G(u)=G_u$.  Independence and~\eqref{eq:app-hyper}, applied to
$C=G_u^{1/2}A$, give
\begin{equation}
 \E_h[X_A^4\mid G_u]
 \le\kappa_h\{\operatorname{tr}(G_uA\Sigma A^\trans)\}^2.
\end{equation}
Because $A\Sigma A^\trans\succeq0$ and $G_u\preceq\Lambda G$, this is at most
$\kappa_h\Lambda^2\{\operatorname{tr}(GA\Sigma A^\trans)\}^2$.
Independence also gives
$\E X_A^2=\operatorname{tr}(GA\Sigma A^\trans)$.  Combine the two identities.
\end{proof}

\begin{lemma}[Singular weighted spectral optimizer]
\label{lem:weighted-svd}
Let $G,\Sigma\succeq0$, let $\Delta_*$ be any compatible matrix, and put
$D_*=G^{1/2}\Delta_*\Sigma^{1/2}$.  Then
\begin{equation}
 \inf_{\rank(M)\le r}
 \|G^{1/2}(M-\Delta_*)\Sigma^{1/2}\|_\F^2
 =\sum_{j>r}\sigma_j(D_*)^2=T_r.
 \label{eq:app-weighted-svd}
\end{equation}
This remains true when $G$ or $\Sigma$ is singular.
\end{lemma}

\begin{proof}
For every feasible $M$, the matrix $G^{1/2}M\Sigma^{1/2}$ has rank at most
$r$, so Eckart--Young--Mirsky gives the lower bound.  Let $(D_*)_r$ be a
rank-$r$ truncated SVD and define
\begin{equation}
 M_r=G^{\dagger/2}(D_*)_r\Sigma^{\dagger/2}.
\end{equation}
The left and right singular vectors of $(D_*)_r$ lie in the supports of $G$
and $\Sigma$, respectively.  Hence
$G^{1/2}M_r\Sigma^{1/2}=(D_*)_r$ and
$\rank(M_r)\le r$.  This candidate attains the tail.
\end{proof}

\begin{proof}[Proof of Theorem~\ref{thm:main-spectral}]
Exact dense-target realizability gives
$d_M(u)=\beta\PiC{n(u)}K(u)(M-\Delta_*)h(u)$.  Independence yields
\begin{equation}
 Q(A)=\|G^{1/2}A\Sigma^{1/2}\|_\F^2.
\end{equation}
By Lemma~\ref{lem:weighted-svd}, every feasible $A=M-\Delta_*$ satisfies
$Q(A)\ge T_r$.  Lemma~\ref{lem:uniform-moment} and
Lemma~\ref{lem:moment-bridge} therefore imply, uniformly over the complete
candidate class,
\begin{equation}
 \E\psi(X_A)
 \ge\frac{\psi(\sqrt{Q(A)})}{1+\Lambda\sqrt{\kappa_h}}
 \ge\frac{\psi(\sqrt{T_r})}{1+\Lambda\sqrt{\kappa_h}}.
\end{equation}
Taking the infimum gives the robust-objective lower bound.

For any nonnegative $X$,
\begin{equation}
 \E\psi_{\rm up}(X)
 \le\min\{\E X^2/4,\sqrt2\,\E X\}
 \le\psi_{\rm up}(\sqrt{\E X^2}).
 \label{eq:app-robust-upper}
\end{equation}
Apply this to the explicit weighted-SVD candidate $M_r$ from
Lemma~\ref{lem:weighted-svd}, for which $Q(M_r-\Delta_*)=T_r$.  Thus
$\Phi_r\le\psi_{\rm up}(\sqrt{T_r})$.  Corollary~\ref{cor:functional-equivalence}
converts the two robust-objective bounds into~\eqref{eq:main-spectral-law}.
\end{proof}

\section{Dependence-allowing rank--KL bounds}
\label{app:dependent}

Without assuming independence, retain the definitions
\begin{equation}
 Q(A)=\E X_A^2,
 \qquad
 S(A)=\|G^{1/2}A\Sigma^{1/2}\|_\F^2,
 \qquad G=\E G(u).
\end{equation}
\begin{proof}[Proof of Theorem~\ref{thm:dependent-general}]
Every feasible $A$ satisfies $S(A)\ge T_r$.  Applying
Lemma~\ref{lem:moment-bridge} and~\eqref{eq:dependent-lower} gives
\begin{equation}
 \E\psi(X_A)
 \ge\frac{\psi(\sqrt{Q(A)})}{1+\sqrt{\kappa_r}}
 \ge\frac{\psi(\sqrt{c_rT_r})}{1+\sqrt{\kappa_r}}.
\end{equation}
This holds for the entire candidate class and survives the infimum.  For the
upper bound, insert $A_r^{\rm svd}$ into~\eqref{eq:app-robust-upper}.  Apply
Corollary~\ref{cor:functional-equivalence}.
\end{proof}

\section{Target-Fisher bounds}
\label{app:fisher}

\begin{proof}[Proof of Theorem~\ref{thm:fisher-comparison}]
Let $q_t=\softmax(z_*+td)$.  Coordinatewise,
\begin{equation}
 e^{-tR(d)}p_{*,i}\le q_{t,i}\le e^{tR(d)}p_{*,i}.
 \label{eq:app-fisher-ratio}
\end{equation}
For any probability vector $q$,
\begin{equation}
 v^\trans H(q)v=\min_c\sum_iq_i(v_i-c)^2.
 \label{eq:app-weighted-variance}
\end{equation}
Consequently, if $\alpha p_i\le q_i\le\gamma p_i$, then
$\alpha H(p)\preceq H(q)\preceq\gamma H(p)$.  The lower inequality follows
by applying the coordinatewise lower bound before minimizing over $c$; for the
upper inequality, evaluate the $q$-weighted variance at the $p$-optimal
centering constant.

Insert~\eqref{eq:app-fisher-ratio} into the Bregman integral
\begin{equation}
 \KL(p_*\|\softmax(z_*+d))
 =\int_0^1(1-t)d^\trans H(q_t)d\,dt.
\end{equation}
The scalar integrals are
\begin{equation}
 \int_0^1(1-t)e^{-tR}dt=\frac{R-1+e^{-R}}{R^2},
 \qquad
 \int_0^1(1-t)e^{tR}dt=\frac{e^R-1-R}{R^2}.
\end{equation}
This proves~\eqref{eq:fisher-comparison}, with the value $1/2$ obtained by
continuity at $R=0$.
\end{proof}

The same integral representations show that $c_-$ is nonincreasing and $c_+$
is nondecreasing on $[0,\infty)$.  For every fixed $t\in[0,1]$,
$(1-t)e^{-tR}$ decreases with $R$, whereas $(1-t)e^{tR}$ increases; integration
preserves these inequalities.  This justifies the monotonicity step below.

For query-only adaptation, define
\begin{equation}
 G_{\F}(u)=\beta^2K(u)^\trans
 [\diag(p_*(u))-p_*(u)p_*(u)^\trans]K(u)
\end{equation}
and $Q_{\F}(A)=\E h^\trans A^\trans G_{\F}(u)Ah$.  Let
$\mathcal C_{r,R_0}$ be the rank-$r$ candidates whose score differences from
the target have range almost surely at most $R_0$, and define
\begin{equation}
 \mathcal E_{r,R_0}
 =\inf_{M\in\mathcal C_{r,R_0}}\E\KL(p_*\|p_M),
 \quad
 T^{\F}_{r,R_0}
 =\inf_{M\in\mathcal C_{r,R_0}}Q_{\F}(M-\Delta_*).
\end{equation}

\begin{corollary}[Constrained Fisher rank--KL bounds]
\label{cor:fisher-constrained}
If $\mathcal C_{r,R_0}$ is nonempty, then
\begin{equation}
 c_-(R_0)T^{\F}_{r,R_0}
 \le\mathcal E_{r,R_0}
 \le c_+(R_0)T^{\F}_{r,R_0}.
\end{equation}
\end{corollary}

\begin{proof}
Apply Theorem~\ref{thm:fisher-comparison} pointwise to every admissible
candidate and use monotonicity of $c_-$ and $c_+$.  Taking the infimum gives
the lower bound.  For the upper bound, apply the inequality to an
$\eta$-minimizing sequence for $T^{\F}_{r,R_0}$ and let $\eta\downarrow0$.
\end{proof}

Under independence, put $G_{\F}=\E G_{\F}(u)$,
$D_{\F}=G_{\F}^{1/2}\Delta_*\Sigma^{1/2}$, and
$T_r^{\F}=\sum_{j>r}\sigma_j(D_{\F})^2$.  Lemma~\ref{lem:weighted-svd}
gives $T_r^{\F}$ as the unrestricted rank-$r$ Fisher tail.  Since
$\mathcal C_{r,R_0}$ is a subset of the unrestricted class,
\begin{equation}
 \mathcal E_{r,R_0}\ge c_-(R_0)T_r^{\F}.
\end{equation}
If the Fisher-weighted SVD optimizer belongs to $\mathcal C_{r,R_0}$, it also
attains the constrained quadratic infimum, yielding the matching upper bound
$\mathcal E_{r,R_0}\le c_+(R_0)T_r^{\F}$.

\section{High-mass lower bounds}
\label{app:highmass}

\begin{lemma}[Conditional KL decomposition]
\label{lem:conditional-kl}
For positive distributions $p,q$ and nonempty $S$, put
$s=p(S)$ and $t=q(S)$.  Then
\begin{equation}
 \KL(p\|q)
 =\operatorname{kl}_{\rm Bern}(s\|t)
  +s\KL(p^S\|q^S)
  +(1-s)\KL(p^{S^c}\|q^{S^c}).
 \label{eq:app-conditional-kl}
\end{equation}
If $p=\softmax(z)$ and $q=\softmax(w)$, then
$p^S=\softmax(z_S)$ and $q^S=\softmax(w_S)$.
\end{lemma}

\begin{proof}
For $i\in S$, write
$\log(p_i/q_i)=\log(s/t)+\log(p_i^S/q_i^S)$ and sum with weights $p_i$.
Repeat on $S^c$.  Conditional softmax follows because the full normalizer
cancels after conditioning.
\end{proof}

\begin{proof}[Proof of Theorem~\ref{thm:high-mass}]
By Lemma~\ref{lem:conditional-kl}, the full KL is at least
$s_*\KL(p_*^S\|p_M^S)$, where $s_*=p_*(S)$.  Apply
Theorem~\ref{thm:global-softmax} to the restricted logits.  Their target floor
is
\begin{equation}
 a_S=\min_{i\in S}\frac{p_{*,i}}{s_*}.
\end{equation}
Thus, for each candidate,
\begin{equation}
 \KL(p_*\|p_M)
 \ge\frac{s_*a_S}{2e^2}
 \psi(\|\PiC{S}(z_M-z_*)_S\|_2).
\end{equation}
Use $s_*\ge1-\delta$ and $a_S\ge a_{\core}$, take expectations, and then take
the infimum over the complete unrestricted rank-$r$ class.
\end{proof}

For the spectral form, let $K_S(u)$ contain the selected key rows and define
\begin{equation}
 G_{\core}(u)=\beta^2K_S(u)^\trans\PiC{S(u)}K_S(u),
 \qquad
 D_{\core}=G_{\core}^{1/2}\Delta_*\Sigma^{1/2},
\end{equation}
where $G_{\core}=\E G_{\core}(u)$.  Under independence, activation
hypercontractivity, and
$G_{\core}(u)\preceq\Lambda_{\core}G_{\core}$, the same proof as
Theorem~\ref{thm:main-spectral} gives
\begin{equation}
 \mathcal E_r
 \ge
 \frac{(1-\delta)a_{\core}}
 {2e^2(1+\Lambda_{\core}\sqrt{\kappa_h})}
 \psi\!\left(\sqrt{T_r^{\core}}\right),
\end{equation}
where
$T_r^{\core}=\sum_{j>r}\sigma_j(D_{\core})^2$.  Under two-sided dependent
geometry
$\lambda_{\core}\bar G_{\core}\preceq G_{\core}(u)
\preceq\Lambda_{\core}\bar G_{\core}$, apply
Theorem~\ref{thm:dependent-general} to obtain
\begin{equation}
 \mathcal E_r
 \ge
 \frac{(1-\delta)a_{\core}}
 {2e^2(1+(\Lambda_{\core}/\lambda_{\core})\sqrt{\kappa_h})}
 \psi\!\left(\sqrt{\lambda_{\core}\bar T_r^{\core}}\right).
\end{equation}

\section{Saturation constructions and lower bound}
\label{app:saturation}

For a boundary target matrix $P_\infty$ and an allowed centered score space
$L$, define its softmax closure rank as the least $r$ for which there is a
sequence $Z_m$ satisfying $\rank(Z_m)\le r$, every column of $Z_m$ lies in
$L$, and columnwise softmax converges to $P_\infty$.

\subsection{The linear-token achievable construction}

Write $k=3q+s$, $s\in\{0,1,2\}$, and choose the least $t$ such that
$n=4^t\ge k+1$.  Identify $\mathbb F_2^{2t}$ with $\mathbb F_4^t$.  The
one-dimensional $\mathbb F_4$ subspaces partition nonzero vectors into
triples $\{a,b,a+b\}$.  Select $q$ complete triples and $s$ characters from
one further triple.  For token row $x$, set
$\chi_a(x)=(-1)^{\langle a,x\rangle}$ and form the normalized character
matrix $K$.  Its columns are centered and orthonormal.  Context $j$ has
$h_j=e_j$, zero pretrained logits, and target logits $TKe_j$.

At every finite $T$, centered-softmax injectivity and $K^\trans K=I$ force an
exact update to equal $TI_k$, hence to have rank $k$.

On a complete triple, write the characters as $x,y,xy$, put
$g=T^{1/2}$ and $\alpha=T^{-3/4}$, and define
\begin{equation}
 b_2=(T,T+g,-T)^\trans,
 \quad b_3=(T,-T,T+g)^\trans,
 \quad b_1=\alpha(b_2+b_3),
 \quad B_T=[b_1\ b_2\ b_3].
 \label{eq:app-triple-block}
\end{equation}
$B_T$ has rank two.  On the first context's target support $x=1$, its scores
are $2T^{1/4}+2T^{-1/4}y$, whereas on $x=-1$ the score is $-2T^{1/4}$.
The within-support spread vanishes and the support gap diverges.  For the
second context the score is
$xT+y(T+g)-xyT$: it equals $T+g$ on $y=1$ and is at most $T-g$ on $y=-1$.
The third context is symmetric.  Scalar $T^{1/4}$ blocks handle leftovers.
The block-diagonal update has rank
$2q+s=k-\lfloor k/3\rfloor$.

For the seven-character strengthening, take the seven nonzero characters of
$\mathbb F_2^3$, set $\rho=T^{-2}$ and $L=T^5$, and define
\begin{equation}
C_\rho=
\begin{bmatrix}
1&\rho^2&0&0&\rho&0&0\\
0&\rho&1&\rho^2&\rho^2&0&0\\
0&0&0&\rho&\rho^2&1&0\\
0&\rho^2&0&\rho^2&0&0&1
\end{bmatrix}
\end{equation}
and
\begin{equation}
F_\rho=
\begin{bmatrix}
0&0&0&0\\
-\rho&-\rho^2&0&-\rho^2\\
0&-\rho&-\rho^2&-\rho^2\\
-1&0&-\rho^2&0\\
0&0&-1&-\rho\\
-\rho^2&-\rho^2&-\rho&0\\
0&-1&0&0\\
-\rho^2&0&0&-1
\end{bmatrix},
\qquad Z_T=LF_\rho C_\rho.
\label{eq:app-seven-factor}
\end{equation}
The factorization gives $\rank(Z_T)\le4$.  For completeness, the product before
the scalar factor $L$ is
{\scriptsize
\begin{equation}
F_\rho C_\rho=
\begin{bmatrix}
0&0&0&0&0&0&0\\
-\rho&-2\rho^3-\rho^4&-\rho^2&-2\rho^4&-\rho^2-\rho^4&0&-\rho^2\\
0&-\rho^2-\rho^4&-\rho&-2\rho^3-\rho^4&-\rho^3-\rho^4&-\rho^2&-\rho^2\\
-1&-\rho^2&0&-\rho^3&-\rho-\rho^4&-\rho^2&0\\
0&-\rho^3&0&-\rho-\rho^3&-\rho^2&-1&-\rho\\
-\rho^2&-\rho^3-\rho^4&-\rho^2&-\rho^2-\rho^4&-2\rho^3-\rho^4&-\rho&0\\
0&-\rho&-1&-\rho^2&-\rho^2&0&0\\
-\rho^2&-\rho^2-\rho^4&0&-\rho^2&-\rho^3&0&-1
\end{bmatrix}.
\label{eq:app-seven-product}
\end{equation}
}
In column $a$, the four rows in the positive halfspace $S_a$ are precisely
those with entries between $-3\rho^3$ and $0$; every remaining entry is at most
$-\rho^2$ for $0<\rho<1/3$.  Multiplying by $L$ gives, for
the positive halfspace $S_a$ of every character $a$,
\begin{equation}
 -3\rho^3L\le Z_T(x,a)\le0\quad(x\in S_a),
 \qquad
 Z_T(x,a)\le-\rho^2L\quad(x\notin S_a).
\end{equation}
Thus the within-support spread is at most $3/T$ and the support gap is at least
$T-3/T$.  Column centering preserves softmax and does not increase rank.

Identify $\mathbb F_2^{3t}$ with $\mathbb F_8^t$ and pack its nonzero vectors
into seven-character one-dimensional $\mathbb F_8$ subspaces.  If
$k=7m+s$, $0\le s<7$, the resulting path has rank at most
$4m+s=k-3\lfloor k/7\rfloor$ and uses $n<8(k+1)$ tokens.

It remains to verify KL convergence rather than only weak convergence of
supports.  For each target halfspace, the target off-support mass is
$\epsilon_T=(1+e^{2T/\sqrt n})^{-1}$.  Let $\eta_T$ be the candidate
off-support mass.  In both constructions, $\eta_T\to0$, within-support KL
vanishes, and the candidate score range is polynomial in $T$.  By
Lemma~\ref{lem:conditional-kl}, the only nontrivial term satisfies
\begin{equation}
 \epsilon_T\log(\epsilon_T/\eta_T)
 \le\epsilon_T[-\log\eta_T]
 =\epsilon_T\operatorname{poly}(T)\longrightarrow0.
\end{equation}
The conditional complement contribution obeys the same bound.  Taking the
better of the triple and septuple packings proves the linear-token statement
in Section~\ref{sec:saturation}.

\subsection{Proof of the exact closure-rank theorem}

\begin{proof}[Proof of Theorem~\ref{thm:exact-closure}]
Write $k=3q+s$, $s\in\{0,1,2\}$.  Work in a binary vector space $V$ whose
size is the least power of two above $2k^2+2$.  We construct
\begin{equation}
 \mathcal A=T_1\mathbin{\dot\cup}\cdots\mathbin{\dot\cup}T_q
 \mathbin{\dot\cup}R,
 \quad
 T_\ell=\{a_\ell,b_\ell,a_\ell+b_\ell\},
 \quad |R|=s,
 \label{eq:app-isolated-packing}
\end{equation}
so these are the only additive triples in $\mathcal A$.  Given a current set
$A_j$, let $F_j=A_j+A_j$.  Choose $u\notin F_j$ and then
$v\notin F_j\cup(u+F_j)\cup\{0,u\}$.  Fewer than $2k^2+2$ vectors are
forbidden, so the construction continues.  At most two leftovers can be
added outside the current sumset.  The number of token rows satisfies
$n<4k^2+8$.

Use normalized Walsh characters $\chi_a$ for $a\in\mathcal A$ as centered
orthonormal key columns.  The finite targets are again $TKe_a$, so every exact
finite-logits update is $TI_k$ and has rank $k$.  Applying the rank-two path
in~\eqref{eq:app-triple-block} independently to each $T_\ell$, with scalar
paths on leftovers, gives closure rank at most $2q+s$.

For the reverse inequality, consider any coefficient sequence $M_m$ whose
score softmaxes converge to the target boundary distributions.  Absorb the
fixed Walsh normalization into $M_m$.  For context $a$, write
\begin{equation}
 f_{m,a}(x)=\sum_{b\in\mathcal A}m_{b,a}^{(m)}\chi_b(x),
 \qquad t_{m,a}=m_{a,a}^{(m)}.
\end{equation}
On $S_a=\{x:\chi_a(x)=1\}$, the restrictions of $\chi_b$ and
$\chi_{b+a}$ agree.  Orthogonality of the restricted characters and convergence
to the uniform target give
\begin{equation}
 m_{b,a}^{(m)}+m_{b+a,a}^{(m)}\longrightarrow0
 \quad\text{for }b\notin\{0,a\}.
 \label{eq:app-restricted-fourier}
\end{equation}
All support scores equal $t_{m,a}+o(1)$, whereas the mean complement score is
$-t_{m,a}$.  The diverging support gap forces $t_{m,a}\to\infty$.

Triple isolation and~\eqref{eq:app-restricted-fourier} make every cross-block
coefficient $o(1)$.  If $\{a,b,c=a+b\}$ is the block containing $a$, put
$s_{m,a}=m_{b,a}^{(m)}$.  Then
$m_{c,a}^{(m)}=-s_{m,a}+o(1)$.  On the negative halfspace,
$\chi_c=-\chi_b$ and both signs occur, so
\begin{equation}
 \max_{S_a^c}f_{m,a}
 =-t_{m,a}+2|s_{m,a}|+o(1).
\end{equation}
The diverging gap implies $|s_{m,a}|\le t_{m,a}+o(t_{m,a})$.

Normalize each column by its positive diagonal coefficient:
\begin{equation}
 C_m=M_m\diag(t_{m,a}^{-1}:a\in\mathcal A).
\end{equation}
This preserves rank.  Diagonal entries become one, cross-block entries
converge to zero, and within-block entries remain bounded.  Pass to a
convergent subsequence.  The limit is block diagonal, with singleton blocks
$[1]$ and triple blocks of the form
\begin{equation}
 B(\alpha,\beta,\gamma)=
 \begin{bmatrix}
 1&\beta&\gamma\\
 \alpha&1&-\gamma\\
 -\alpha&-\beta&1
 \end{bmatrix}.
 \label{eq:app-triple-limit}
\end{equation}
Every real block~\eqref{eq:app-triple-limit} has rank at least two.  If it had
rank one, its $2\times2$ minors would force
$\alpha\beta=1$, $-\alpha\gamma=1$, and $\beta\gamma=1$.  The first and third
equalities imply $\alpha=\gamma$, contradicting $-\alpha\gamma=1$ over the
reals.  The limiting normalized matrix therefore has rank at least
$2q+s=k-\lfloor k/3\rfloor$.  Since the set of matrices of rank at most $r$
is closed, every approximating sequence has rank at least $2q+s$.  This
matches the construction.
\end{proof}

\section{Fused multi-head proofs}
\label{app:multihead}

For nonnegative numbers $x_1,\ldots,x_H$, put
$s=(\sum_hx_h^2)^{1/2}$.  We first prove
\begin{equation}
 \psi\!\left(\sqrt{\sum_hx_h^2}\right)
 \le\sum_h\psi(x_h)
 \le\sqrt H\,\psi\!\left(\sqrt{\sum_hx_h^2}\right)
 \label{eq:app-psi-aggregate}
\end{equation}
for all $x_h\ge0$.  If $s\le1$, every $x_h\le1$ and
$\sum_h\psi(x_h)=s^2=\psi(s)$.  If $s>1$, then
$\psi(x_h)\ge x_h^2/s$: for $x_h\le1$ this follows from $s\ge1$, and for
$x_h>1$ it follows from $x_h\le s$.  Summing gives the left inequality.
Moreover, $\psi(x_h)\le x_h$, so Cauchy--Schwarz gives
$\sum_h\psi(x_h)\le\sum_hx_h\le\sqrt Hs=\sqrt H\psi(s)$.

\begin{proof}[Proof of Theorem~\ref{thm:multihead}]
Applying Theorem~\ref{thm:global-softmax} headwise and summing, then using
Equation~\eqref{eq:app-psi-aggregate}, gives the lower bound
\begin{equation}
 \frac{a_{\min}}{2e^2}\psi(X_{\MH})
 \le\sum_hD_h,
 \qquad X_{\MH}^2=\sum_hX_h^2.
\end{equation}
For the upper bound, the headwise inequalities give directly
\begin{equation}
 \sum_hD_h
 \le\sum_h\min\{X_h^2/4,\sqrt2X_h\}
 \le\min\{X_{\MH}^2/4,\sqrt{2H}X_{\MH}\}.
 \label{eq:app-mh-upper}
\end{equation}
Both constants in the final aggregation inequality are attained: the
quadratic branch when all $X_h$ are small, and the linear branch when the
errors are equal across heads.
The block definitions in~\eqref{eq:multihead-definitions} and independence
give
\begin{equation}
 \E X_{\MH}^2
 =\|G_{\MH}^{1/2}(M-\Delta_*^{\MH})\Sigma^{1/2}\|_\F^2.
\end{equation}
Weighted Eckart--Young--Mirsky makes the infimum of this quadratic objective
equal to $T_r^{\MH}$.  The leverage and fourth-moment assumptions imply
$\E X_{\MH}^4\le\Lambda_{\MH}^2\kappa_h(\E X_{\MH}^2)^2$ by the argument of
Lemma~\ref{lem:uniform-moment}.  Lemma~\ref{lem:moment-bridge} therefore gives
the robust lower bound
\begin{equation}
 \E\psi(X_{\MH})\ge
 \frac{\psi(\sqrt{T_r^{\MH}})}{1+\Lambda_{\MH}\sqrt{\kappa_h}}.
\end{equation}
For the weighted-SVD candidate, take expectations in
Equation~\eqref{eq:app-mh-upper}, use Cauchy--Schwarz on $\E X_{\MH}$, and
substitute $\E X_{\MH}^2=T_r^{\MH}$.  This gives the stated upper bound and
completes the proof of~\eqref{eq:multihead-law}.
\end{proof}

For independently parameterized head adapters, the feasible set is a Cartesian
product once integer ranks $r_h$ are fixed, so
\begin{equation}
 \inf_{\substack{\sum_hr_h\le R\\\rank(M_h)\le r_h}}
 \sum_h\E D_h
 =\min_{\sum_hr_h\le R}\sum_h\mathcal E_{h,r_h}.
\end{equation}
This proves the allocation identity.

Finally, let $y_{M,h}$ and $y_{*,h}$ be the value-weighted head outputs, let
$\mathcal D_h(u)$ be the diameter of the head's value vectors, and let $O_h$
be its output-projection block.  Total variation and Pinsker give
\begin{equation}
 \|y_{M,h}-y_{*,h}\|_2
 \le \mathcal D_h(u)\operatorname{TV}(p_{M,h},p_{*,h})
 \le \mathcal D_h(u)\sqrt{D_h/2}.
\end{equation}
By Cauchy--Schwarz across heads,
\begin{equation}
 \left\|\sum_hO_h(y_{M,h}-y_{*,h})\right\|_2^2
 \le\frac H2\sum_h\|O_h\|_{\op}^2\mathcal D_h(u)^2D_h(u).
\end{equation}

\section{Joint query/key proofs}
\label{app:jointqk}

For ordinary dot-product attention, define
\begin{equation}
 C(A,B)=K_0^\trans A+B^\trans(Q_0+A).
\end{equation}
The two summands have ranks at most $\rank(A)$ and $\rank(B)$, proving
Equation~\eqref{eq:qk-rank}.  The bound is sharp: for $p=r$, $d=2r$, take
\begin{equation}
 K_0=A=[I_r\ 0],
 \qquad Q_0=B=[0\ I_r].
\end{equation}
Then $C(A,B)$ has rank $2r$.

Every joint candidate also satisfies
\begin{equation}
 \rank(K_0^\trans Q_0+C(A,B))
 =\rank((K_0+B)^\trans(Q_0+A))\le p.
 \label{eq:app-width}
\end{equation}
Thus the actual joint class is contained in the effective rank-$s$ class,
$s=r_Q+r_K$, and in its width-aware subclass.

Let
\begin{equation}
 d_{A,B}(u)=\PiC{n(u)}\beta X(u)[C(A,B)-C_*]h(u)
\end{equation}
and let $\Psi^{\QK}_{r_Q,r_K}$ be the infimum of
$\E\psi(\|d_{A,B}\|_2)$ over the actual factor class.  Applying
the lower half of Theorem~\ref{thm:global-softmax} to each candidate gives
\begin{equation}
 \frac{a}{2e^2}\Psi^{\QK}_{r_Q,r_K}
 \le\mathcal E^{\QK}_{r_Q,r_K}.
 \label{eq:app-qk-robust}
\end{equation}

\begin{proof}[Proof of Theorem~\ref{thm:joint-qk}]
By rank containment, every effective error
$C(A,B)-C_*$ is a rank-$s$ candidate relative to $C_*$.  Weighted
Eckart--Young--Mirsky therefore gives
\begin{equation}
 \E\|d_{A,B}(u)\|_2^2\ge T_s^{\eff}
\end{equation}
for every feasible pair.  The assumed fourth-moment condition and
Lemma~\ref{lem:moment-bridge} imply
\begin{equation}
 \E\psi(\|d_{A,B}\|_2)
 \ge\frac{\psi(\sqrt{T_s^{\eff}})}{1+\sqrt\kappa}.
\end{equation}
Take the infimum and use~\eqref{eq:app-qk-robust} for the lower bound.

Let $\mathcal S_s^{\rm svd}$ be the nonempty set of weighted rank-$s$
optimizers and define
\begin{equation}
 \rho_{r_Q,r_K}
 =\inf_{C_s\in\mathcal S_s^{\rm svd}}
   \inf_{\substack{\rank(A)\le r_Q\\\rank(B)\le r_K}}
 \|G_X^{1/2}[C(A,B)-C_s]\Sigma^{1/2}\|_\F.
\end{equation}
Choose jointly $\eta$-minimizing sequences in these two infima.  The weighted
seminorm triangle inequality gives
\begin{equation}
 \left\{
 \E\|d_{A,B}(u)\|_2^2
 \right\}^{1/2}
 \le\sqrt{T_s^{\eff}}+\rho_{r_Q,r_K}+o(1).
\end{equation}
Apply the pointwise global upper bound and Equation~\eqref{eq:app-robust-upper},
then let $\eta\downarrow0$.
\end{proof}

The realizability price vanishes in several checkable cases.  If
$\rank(C)\le r_Q$ and $\col(C)\subseteq\col(K_0^\trans)$, then
\begin{equation}
 A=(K_0^\trans)^\dagger C,
 \qquad C(A,0)=C,
 \qquad \rank(A)\le\rank(C).
\end{equation}
The symmetric statement for keys uses
$B^\trans=CQ_0^\dagger$ when $\row(C)\subseteq\row(Q_0)$.  More generally,
suppose $C=C_Q+C_K$, choose $A_Q$ with
$K_0^\trans A_Q=C_Q$ and $\rank(A_Q)\le r_Q$, and assume
$\row(C_K)\subseteq\row(Q_0+A_Q)$ with $\rank(C_K)\le r_K$.  Then
\begin{equation}
 B_K^\trans=C_K(Q_0+A_Q)^\dagger
\end{equation}
has rank at most $r_K$ and satisfies
$C(A_Q,B_K)=C_Q+C_K$.  These projector identities prove the one-sided and
sequential realization claims.

An always-valid factor-dependent upper bound follows from
\begin{equation}
 C(A,B)-C(A_*,B_*)
 =(K_0+B)^\trans(A-A_*)+(B-B_*)^\trans(Q_0+A_*).
\end{equation}
Submultiplicativity yields
\begin{align}
 &\|G_X^{1/2}[C(A,B)-C_*]\Sigma^{1/2}\|_\F\\
 &\quad\le
 \|G_X^{1/2}(K_0+B)^\trans\|_{\op}
 \|(A-A_*)\Sigma^{1/2}\|_\F\\
 &\qquad+
 \|(Q_0+A_*)\Sigma^{1/2}\|_{\op}
 \|G_X^{1/2}(B-B_*)^\trans\|_\F.
\end{align}
Combining this with~\eqref{eq:app-robust-upper} and the global softmax upper
bound gives the constructive target-factor upper bound.  Weighted truncated
SVDs of the actual target factors make the two Frobenius errors equal to their
corresponding spectral tails.

For fused multi-head joint adaptation, every head block satisfies
$\rank C_h(A_h,B_h)\le s$.  Applying weighted Eckart--Young headwise and then
the aggregate moment bridge gives
\begin{equation}
 \mathcal E^{\MH,\QK}_{r_Q,r_K}
 \ge\frac{a_{\min}}{2e^2(1+\sqrt{\kappa_{\MH}})}
 \psi\!\left(\sqrt{\sum_hT_{h,s}^{\eff}}\right).
\end{equation}
This bound deliberately grants every head the full allowance $s$ and is
therefore safe but potentially loose.